\RequirePackage{xcolor}
\documentclass{article}

\usepackage{times}
\usepackage{graphicx} 
\usepackage{subfigure} 

\usepackage{natbib}

\usepackage{algorithm}
\usepackage{algorithmic}

\usepackage{hyperref}
\usepackage{dsfont}
\usepackage[capitalize,noabbrev]{cleveref}

\makeatletter
\providecommand\theHALG@line{\thealgorithm.\arabic{ALG@line}}
\makeatother

\usepackage[accepted]{icml2017}

\usepackage{booktabs}       
\usepackage{amsfonts}       
\usepackage{amsmath,amssymb,amsthm,amsfonts,latexsym}        
\usepackage{url} 
\usepackage{graphicx}
\usepackage{bm}
\usepackage{subfigure}
\usepackage{mathtools}
\usepackage{xcolor}
\usepackage{setspace}

\newcommand{\so}{u}
\newcommand{\tar}{v}
\newcommand{\p}{p}
\newcommand{\ptrue}{\mbox {$p^{*}$}}

\newcommand{\pucb}{\mbox {$\bar{p}$}}
\newcommand{\CB}{z}

\newcommand{\vecy}{\mbox {$\bm{y}$}}

\newcommand{\vecw}{\mbox {$\bm{\theta}$}}

\newcommand{\vecx}{\mbox {$\bm{x}$}}
\newcommand{\vecb}{\mbox {$\bm{b}$}}
\newcommand{\vecD}{\mbox {$\bm{D}$}}

\newcommand{\Sopt}{\mbox {$S^{*}$}}

\newcommand{\spr}{F}

\newcommand{\seeds}{\mathcal{S}}
\newcommand{\tcS}{\tilde{\mathcal{S}}}

\newcommand{\dilinucb}{{\tt DILinUCB}}
\newcommand{\oracle}{{\tt ORACLE}}

\DeclareMathOperator{\Proj}{Proj}
\DeclareMathOperator{\Tr}{Tr}

\DeclareMathOperator{\vect}{vec}

\newtheorem{theorem}{Theorem}
\newtheorem{lemma}{Lemma}

\makeatletter
\def\thm@space@setup{%
  \thm@preskip=\parskip \thm@postskip=0pt
}
\makeatother

\DeclareMathOperator*{\argmax}{arg\,max}
\DeclareMathOperator*{\argmin}{arg\,min}

\renewcommand{\d}[1]{\ensuremath{\operatorname{d}\!{#1}}}

\newtheorem{assumption}{Assumption}

\newcommand{\E}{\mathbb{E}}

\newcommand{\cC}{\mathcal{C}}
\newcommand{\cD}{\mathcal{D}}
\newcommand{\cE}{\mathcal{E}}

\newcommand{\cF}{\mathcal{F}}
\newcommand{\cG}{\mathcal{G}}
\newcommand{\cH}{\mathcal{H}}

\newcommand{\cS}{\mathcal{S}}

\newcommand{\cV}{\mathcal{V}}

\newcommand{\bw}{{\bf w}}

\renewcommand{\epsilon}{\varepsilon}
\renewcommand{\hat}{\widehat}
\renewcommand{\tilde}{\widetilde}
\renewcommand{\bar}{\overline}

\newcommand{\nothere}[1]{}

\usepackage{xspace}

\usepackage[colorinlistoftodos, textwidth=30mm, shadow, textsize=small]{todonotes}
\definecolor{babyblue}{rgb}{0.54, 0.81, 0.94}
\definecolor{citrine}{rgb}{0.89, 0.82, 0.04}
\definecolor{misocolor}{rgb}{0.16,0.27,0.86}

\icmltitlerunning{Model-Independent Online Learning for Influence Maximization}

\begin{document} 

\twocolumn[
\icmltitle{Model-Independent Online Learning for Influence Maximization}



\icmlsetsymbol{equal}{*}

\begin{icmlauthorlist}
\icmlauthor{Sharan Vaswani}{ubc}
\icmlauthor{Branislav Kveton}{adobe}
\icmlauthor{Zheng Wen}{adobe}
\icmlauthor{Mohammad Ghavamzadeh}{google}
\icmlauthor{Laks V.S. Lakshmanan}{ubc}
\icmlauthor{Mark Schmidt}{ubc}
\end{icmlauthorlist}

\icmlaffiliation{ubc}{University of British Columbia}
\icmlaffiliation{adobe}{Adobe Research}
\icmlaffiliation{google}{DeepMind (The work was done when the author was with Adobe Research)}
\icmlcorrespondingauthor{Sharan Vaswani}{sharanv@cs.ubc.ca}

\icmlkeywords{viral marketing, influence maximization, linear bandits, upper confidence bound}

\vskip 0.3in
]

\printAffiliationsAndNotice{}  

\setstretch{0.935}

\begin{abstract}
We consider \emph{influence maximization} (IM) in social networks, which is the problem of maximizing the number of users that become aware of a product by selecting a set of ``seed'' users to expose the product to. While prior work assumes a known model of information diffusion, we propose a novel parametrization that not only makes our framework agnostic to the underlying diffusion model, but also statistically efficient to learn from data. We give a corresponding monotone, submodular surrogate function, and show that it is a good approximation to the original IM objective. We also consider the case of a new marketer looking to exploit an existing social network, while simultaneously learning the factors governing information propagation. For this, we propose a pairwise-influence semi-bandit feedback model and develop a LinUCB-based bandit algorithm. Our model-independent analysis shows that our regret bound has a better (as compared to previous work) dependence on the size of the network. Experimental evaluation suggests that our framework is robust to the underlying diffusion model and can efficiently learn a near-optimal solution.
\end{abstract}

\vspace*{-4ex}
\section{Introduction}
\label{sec:introduction}
\vspace*{-1.5ex}
The aim of viral marketing is to spread awareness about a specific product via word-of-mouth information propagation over a social network. More precisely, marketers (agents) aim to select a fixed number of influential users (called \emph{seeds}) and provide them with free products or discounts. They assume that these users will influence their neighbours and, transitively, other users in the social network to adopt the product. This will thus result in information propagating across the network as more users adopt or become aware of the product. The marketer has a budget on the number of free products and must choose seeds in order to maximize the \emph{influence spread} which is the expected number of users that become aware of the product. This problem is referred to as \emph{influence maximization} (IM). 

Existing solutions to the IM problem require as input, the underlying diffusion model which describes how information propagates through the network. The IM problem has been studied under various probabilistic diffusion models such as independent cascade (IC) and linear threshold (LT) models~\citep{kempe2003maximizing}. Under these common models, there has been substantial work on developing efficient heuristics and approximation algorithms~\cite{chen2009efficient,leskovec2007cost,goyal2011simpath,goyal2011data,Tang2014Influence,Tang2015TIM+}. 

Unfortunately, knowledge of the underlying diffusion model and its parameters is essential for the existing IM algorithms to perform well. For example,~\citet{Du2014} empirically showed that misspecification of the diffusion model can lead to choosing bad seeds and consequently to a low spread. In practice, it is not clear how to choose from amongst the increasing number of plausible diffusion models~\cite{kempe2003maximizing,gomez2012influence,li2013influence}. Even if we are able to choose a diffusion model according to some prior information, the number of parameters for these models scales with the size of the network (for example, it is equal to the number of edges for both the IC and LT models) and it is not clear how to set these. \citet{goyal2011data} showed that even when assuming the IC or LT model, correct knowledge of the model parameters is critical to choosing good seeds that lead to a large spread. Some papers try to learn these parameters from past propagation data~\cite{saito2008prediction,goyal2010learning,netrapalli2012learning}. However in practice, such data is hard to obtain and the large number of parameters makes this learning challenging. 

To overcome these difficulties, we propose a novel parametrization for the IM problem in terms of \emph{pairwise reachability probabilities} (\cref{sec:im}). This parametrization depends only on the state of the network after the information diffusion has taken place. Since it does not depend on \emph{how} information diffuses, it is agnostic to the underlying diffusion model. To select seeds based on these reachability probabilities, we propose a monotone and submodular surrogate objective function based on the notion of \emph{maximum reachability} (\cref{sec:approximation}). Our surrogate function can be optimized efficiently and is a a good approximation to the IM objective. We theoretically bound the quality of this approximation. Our parametrization may be of independent interest to the IM community.

Next, we consider learning how to choose good seeds in an online setting. Specifically, we focus on the case of a new marketer looking to exploit an existing network to market their product. They need to choose a good seed set, while simultaneously learning the factors affecting information propagation. This motivates the learning framework of IM semi-bandits~\cite{vaswani2015influence, chen2016combinatorial, wen2016influence}. In these works, the marketer performs IM over multiple ``rounds'' and learns about the factors governing the diffusion on the fly. Each round corresponds to an IM attempt for the same or similar products. Each attempt incurs a loss in the influence spread (measured in terms of \emph{cumulative regret}) because of the lack of knowledge about the diffusion process. The aim is to minimize the cumulative regret incurred across multiple such rounds. This leads to the classic exploration-exploitation trade-off where the marketer must either choose seeds that either improve their knowledge about the diffusion process (``exploration'') or find a seed set that leads to a large expected spread (``exploitation''). Note that all previous works on IM semi-bandits assume the IC model. 
%

We propose a novel semi-bandit feedback model based on pairwise influence (\cref{sec:semi-bandit}). Our feedback model is weaker than the edge-level feedback proposed in~\cite{chen2016combinatorial, wen2016influence}. Under this feedback, we formulate IM semi-bandit as a linear bandit problem and propose a scalable LinUCB-based algorithm (\cref{sec:ucb}). We bound the cumulative regret of this algorithm (Section~\ref{sec:bound}) and show that our regret bound has the optimal dependence on the time horizon, is linear in the cardinality of the seed set, and as compared to the previous literature, has a better dependence on the size of the network. In \cref{sec:implementation}, we describe how to construct features based on the graph Laplacian eigenbasis and describe a practical implementation of our algorithm. Finally, in \cref{sec:experiments}, we empirically evaluate our proposed algorithm on a real-world network and show that it is statistically efficient and robust to the underlying diffusion model.

\section{Influence Maximization}
\label{sec:im} 
The IM problem is characterized by the triple $\left( \cG,  \cC, \cD \right )$, where $\cG$ is a directed graph encoding the topology of the social network, $\cC$ is the collection of feasible seed sets, and $\cD$ is the underlying diffusion model. Specifically, $\cG = \left( \cV, \cE \right)$, where $\cV=\left \{1,2,\ldots, n\right \}$ and $\cE$ are the node and edge sets of $\cG$, with cardinalities $n = \vert \cV \vert$ and $m = \vert \cE \vert$, respectively. 
The collection of feasible seed sets $\cC$ is determined by a \emph{cardinality constraint} on the sets and possibly some \emph{combinatorial constraints} (e.g. matroid constraints) that rule out some subsets of $\cV$. This implies that $\cC \subseteq \{\cS \subseteq \cV: |\cS| \leq K\}$, for some $K \leq n$. The diffusion model $\cD$ specifies the stochastic process under which influence is propagated across the social network once a seed set $\cS \in \cC$ is selected. Without loss of generality, we assume that all stochasticity in $\cD$ is encoded in a random vector $\bw$, referred to as the \emph{diffusion random vector}. Note that throughout this paper, we denote vectors in bold case. We assume that each diffusion has a corresponding $\bw$ sampled independently from an underlying probability distribution $\mathbb{P}$ specific to the diffusion model. For the widely-used models IC and LT, $\bw$ is an $m$-dimensional binary vector encoding edge activations for all the edges in $\cE$, and $\mathbb{P}$ is parametrized by $m$ \emph{influence probabilities}, one for each edge. Once $\bw$ is sampled, we use $\cD(\bw)$ to refer to the particular realization of the diffusion model $\cD$. Note that by definition, $\cD(\bw)$ is deterministic, conditioned on $\bw$. 
 
Given the above definitions, an IM attempt can be described as: the marketer first chooses a seed set $\cS \in \cC$ and then nature independently samples a diffusion random vector $\bw \sim \mathbb{P}$. Note that the influenced nodes in the diffusion are completely determined by $\cS$ and $\cD(\bw)$. We use the indicator $\mathds{1}\big(\cS,v,\cD(\bw)\big) \in \{0, 1 \}$ to denote if the node $v$ is influenced under the seed set $\cS$ and the particular realization $\cD(\bw)$. For a given $(\cG, \cD)$, once a seed set $\cS \subseteq \cC$ is chosen, for each node $v \in \cV$, we use $\spr(\cS, v)$ to denote the probability that $v$ is influenced under the seed set $\cS$, i.e.,
\begin{flalign}
\spr(\cS, v) =  \mathbb{E} \left[  \mathds{1}\big(\cS, v, \cD (\bw)\big) \middle | \cS \right]
\end{flalign}
where the expectation is over all possible realizations $\cD (\bw)$. We denote by $\spr(\cS) = \sum_{v \in \cV} \spr(\cS, v)$, the expected number of nodes that are influenced when the seed set $\cS$ is chosen. The aim of the IM problem is to maximize $\spr(\cS)$ subject to the constraint $\cS \in \cC$, i.e., to find $\cS^\ast \in \argmax_{\mathcal{S} \in \cC} \spr(\cS)$. Although IM is an NP-hard problem in general, under common diffusion models such as IC and LT, the objective function $\spr(\cS)$ is monotone and submodular, and thus, a near-optimal solution can be computed in polynomial time using a greedy algorithm~\cite{nemhauser1978analysis}.
In this work, we assume that $\cD$ is any diffusion model satisfying the following monotonicity assumption:
\begin{assumption}
\label{assum:monotone}
For any $v \in \cV$ and any subsets $\cS_1 \subseteq \cS_2 \subseteq \cV$, if $\spr(\cS_1, v) \leq \spr(\cS_2, v)$, then $\spr(\cS, v)$ is monotone in $\cS$. 
\end{assumption}
\vspace*{-0.5 ex}
Note that all progressive diffusion models (models where once the user is influenced, they can not change their state), including those in~\cite{kempe2003maximizing,gomez2012influence,li2013influence} satisfy Assumption~\ref{assum:monotone}. 
\section{Surrogate Objective}
\label{sec:approximation}
We now motivate and propose a surrogate objective for the IM problem based on the notion of \emph{maximal pairwise reachability}. We start by defining some useful notation. For any set $\cS \subseteq \cV$ and any set of ``pairwise probabilities" $p: \cV \times \cV \rightarrow [0,1]$, for all nodes $v \in \cV$, we define
\begin{flalign}
f(\cS, v, \p) = \textstyle \max_{u \in \cS}\;p_{u,v}
\label{eq:oracle}
\end{flalign}
where $p_{u,v}$ is the pairwise probability associated with the ordered node pair $(u,v)$. We further define $f(\cS,\p) = \sum_{v \in \cV} f(\cS, v,\p)$. Note that for all  $p$, $f(\cS,\p)$ is always monotone and submodular in $\cS$~\cite{krause2012submodular}.

For any pair of nodes $u, v \in \cV$, we define the \emph{pairwise reachability} from $u$ to $v$ as $p^*_{u,v} = \spr(\{ u \}, v)$, i.e.,~the probability that $v$ will be influenced, if $u$ is the only seed node under graph $\cG$ and diffusion model $\cD$. Throughout this paper, we use ``source node'' and ``seed'' interchangeably and refer to the nodes \emph{not} in the seed set $\cS$ as ``target'' nodes. We define $f(\cS, v, \p^*) = \max_{u \in \cS} p^*_{u,v}$ as the \emph{maximal pairwise reachability} from the seed set $\cS$ to the target node $v$. 

Our proposed surrogate objective for the IM problem is $f(\cS, \p^*) = \sum_{v \in \cV} f(\cS, v, \p^*)$. Based on this objective, an approximate solution $\tcS$ to the IM problem can be obtained by maximizing $f(\cS,\ptrue)$ under the constraint $\cS \in \cC$,
\begin{align}
\tcS \in \textstyle \argmax_{\cS \in \cC} f(\cS,\ptrue)
\end{align}
Recall that $\seeds^*$ is the optimal solution to the IM problem. To quantify the quality of the surrogate, we define the \emph{surrogate approximation factor} as $\rho=f(\tcS,\ptrue)/ \spr(\seeds^*)$. The following theorem, (proved in~\cref{sec:mpr_appr}) states that we can obtain the following upper and lower bounds on $\rho$:
\begin{theorem} 
For any graph $\cG$, seed set $\cS \in \cC$, and diffusion model $\cD$ satisfying Assumption~\ref{assum:monotone}, 
\vspace*{-1ex}
\begin{itemize}
\itemsep0em
\item[1] $f(\cS,\ptrue) \leq \spr(\cS)$, 
\item[2] If $\spr(\cS)$ is submodular in $\cS$, then $1/K \leq \rho \leq 1$.  
\end{itemize}
\label{thm:mpr_appr}
\end{theorem}
\vspace*{-2ex}
The above theorem implies that for any progressive model satisfying Assumption~\ref{assum:monotone}, maximizing $f(\cS,\ptrue)$ is  equivalent to maximizing a lower-bound on the true spread $\spr(\cS)$. For both IC and LT models, $\spr(\cS)$ is both monotone and submodular, and the approximation factor can be bounded from below by $1/K$. In~\cref{sec:experiments}, we empirically show that in cases of practical interest, $f(\cS,\ptrue)$ is a good approximation to $\spr(\cS)$ and that $\rho$ is much larger than $1/K$.

Finally, note that solving $\tcS \in \argmax_{\cS \in \cC} f(\cS,\ptrue)$ exactly might be computationally intractable and thus we need to compute a near-optimal solution based on an approximation algorithm. In this paper, we refer to such approximation algorithms as \emph{oracles} to distinguish them from learning algorithms. Let $\oracle$ be a specific oracle and let $\hat{\cS} \stackrel{\Delta}{=} \oracle(\cG, \cC, p)$ be the seed set output by it. For any $\alpha \in [0,1]$, we say that $\oracle$ is an $\alpha$-approximation algorithm if for all $p: \cV \times \cV \rightarrow [0,1]$, $f (\hat{\cS}, p ) \geq \alpha \max_{\cS \in \cC} f(\cS, p)$. For our particular case, since $f(\cS,\ptrue)$ is submodular, a valid oracle is the greedy algorithm which gives an $\alpha = 1 -1/e$ approximation~\cite{nemhauser1978analysis}. Hence, given the knowledge of $\ptrue$, we can obtain an approcimate solution to the IM problem without knowing the exact underlying diffusion model. 
\section{Influence Maximization Semi-Bandits}
\label{sec:semi-bandit}
We now focus on the case of a new marketer trying to learn the pairwise reachabilities by repeatedly interacting with the network. We describe the observable feedback (\cref{sec:feedback}) and the learning framework (\cref{sec:lin-gen}). 
\subsection{Influence Maximization Semi-Bandits}
In an influence maximization semi-bandit problem, the agent (marketer) knows both $\cG$ and $\cC$, but does not know the diffusion model $\cD$. Specifically, the agent knows neither the model of $\cD$, for instance whether $\cD$ is the IC or LT model; nor its parameters, for instance the influence probabilities in the IC or LT model. Consider a scenario in which the agent interacts with the social network for $T$ rounds. 
%
At each round  $t\in\{1,\ldots,T\}$, the agent first chooses a seed set $\cS_t \in \cC$ based on its prior knowledge and past observations and then nature independently samples a diffusion random vector $\bw_t \sim \mathbb{P}$. Influence thus diffuses in the social network from $\cS_t$ according to $\cD(\bw_t)$.
%
%
The agent's reward at round $t$ is the number of the influenced nodes
\begin{equation*}
r_t =\textstyle \sum_{v \in \cV} \mathds{1} \big( \cS_t, v, \cD(\bw_t) \big).
\end{equation*}
Recall that by definition, $\mathbb{E} \left[ r_t \middle | \cS_t, \cD(\bw_t) \right] = \spr\left( \cS_t \right)$. After each such IM attempt, the agent observes the \emph{pairwise influence feedback} (described next) and uses it to improve the subsequent IM attempts. The agent's objective is to maximize the expected cumulative reward across the $T$ rounds, i.e.,~to maximize $\mathbb{E} \left[ \sum_{t=1}^T r_t \right]$. This is equivalent to minimizing the cumulative regret defined subsequently. 
\subsection{Pairwise Influence Feedback Model}
\label{sec:feedback}
We propose a novel IM semi-bandit feedback model referred to as \emph{pairwise influence feedback}. Under this feedback model, at the end of each round $t$, the agent observes $\mathds{1} \big(\{ u\}, v, \cD(\bw_t)\big)$ for all $u \in \cS_t$ and all $v \in \cV$. In other words, it observes whether or not $v$ would be influenced, if the agent selects $\cS = \left \{ u \right \}$ as the seed set under the diffusion model $\cD(\bw_t)$. 
This form of semi-bandit feedback is plausible in most IM scenarios. For example, on sites like Facebook, we can identify the user who influenced another user to ``share'' or ``like'' an article, and thus, can transitively trace the propagation to the seed which started the diffusion. Note that our assumption is strictly weaker than (and implied by) edge level semi-bandit feedback~\cite{chen2016combinatorial, wen2016influence}: from edge level feedback, we can identify the edges along which the diffusion travelled, and thus, determine whether a particular source node is responsible for activating a target node. However, from pairwise feedback, it is impossible to infer a unique edge level feedback. 
 \vspace*{-2ex}
\subsection{Linear Generalization}
\label{sec:lin-gen}
Parametrizing the problem in terms of reachability probabilities results in $O(n^2)$ parameters that need to be learned. Without any structural assumptions, this becomes intractable for large networks. To develop statistically efficient algorithms for large-scale IM semi-bandits, we make a linear generalization assumption similar to~\cite{wen2015efficient, wen2016influence}. Assume that each node $v \in \cV$ is associated with two vectors of dimension~$d$, the seed (source) weight $ \vecw^{*}_{\tar} \in \Re^d$ and the target feature $\vecx_{v} \in \Re^d$. We assume that the target feature $\vecx_{\tar}$ is known, whereas $ \vecw^{*}_{\tar}$ is unknown and needs to be learned. The linear generalization assumption is stated as:
\begin{assumption}
For all $u,v \in \cV$, $p^*_{\so,\tar}$ can be ``well approximated" by the inner product of $ \vecw^{*}_{\so}$ and $\vecx_{\tar}$, i.e., 
\begin{align*}
p^*_{\so,\tar} \approx \langle \vecw^{*}_{\so}, \vecx_{\tar} \rangle \stackrel{\Delta}{=}  \vecx_{\tar}^\top  \vecw^{*}_{\so}
\end{align*}
\label{assump:lg}
\end{assumption}
\vspace*{-3ex}
Note that for the \emph{tabular case} (the case without generalization across $p^*_{\so,\tar}$), we can always choose $\vecx_{\tar}=e_{\tar} \in \Re^n$ and $ \vecw^{*}_{\so} =\left [p^*_{\so,1}, \ldots, p^*_{\so,n} \right]^T$, where $e_{\tar}$ is an $n$-dimensional indicator vector with the $\tar$-th element equal to $1$ and all other elements equal to $0$. However, in this case $d=n$, which is not desirable. Constructing target features when $d \ll n$ is non-trivial. We discuss a feature construction approach based on the unweighted graph Laplacian in \cref{sec:implementation}. We use matrix $X \in \Re^{d \times n}$ to encode the target features. Specifically, for $\tar = 1,\ldots, n$, the $\tar$-th column of $X$ is set as $\vecx_{\tar}$. Note that $X=I \in \Re^{n \times n}$ in the tabular case.

Finally, note that under Assumption~\ref{assump:lg}, estimating the reachability probabilities becomes equivalent to estimating $n$ (one for each source) $d$-dimensional weight vectors. This implies that Assumption~\ref{assump:lg} reduces the number of parameters to learn from $O(n^2)$ to $O(d n)$, and thus, is important for developing statistically efficient algorithms for large-scale IM semi-bandits.


\subsection{Performance Metric}
We benchmark the performance of an IM semi-bandit algorithm by comparing its spread against the attainable influence assuming perfect knowledge of $\cD$. Since it is NP-hard to compute the optimal seed set even when with perfect knowledge, similar to~\cite{wen2016influence,chen2016combinatorial}, we measure the performance of an IM
semi-bandit algorithm by \emph{scaled cumulative regret}. Specifically, if $\cS_{t}$ is the seed set selected by the IM semi-bandit algorithm at round $t$, for any $\kappa \in (0,1)$, the $\kappa$-scaled cumulative regret $R^{\kappa}(T)$ in the first $T$ rounds is defined as
\begin{flalign}
R^{\kappa}(T) = T \cdot \spr(\cS^*) - \frac{1}{\kappa} \E \left[ \textstyle \sum_{t = 1}^{T} \spr( \cS_{t} ) \right].
\label{eq:cumu-regret}
\end{flalign}

\section{Algorithm}
\label{sec:ucb}
In this section, we propose a LinUCB-based IM semi-bandit algorithm, called \emph{diffusion-independent LinUCB} ($\dilinucb$), whose pseudocode is in \cref{algo:ucb}. As its name suggests, $\dilinucb$ is applicable to IM semi-bandits with any diffusion model $\cD$ satisfying Assumption~]ref{assum:monotone}. The only requirement to apply $\dilinucb$ is that the IM semi-bandit provides the pairwise influence feedback described in \cref{sec:feedback}.

\begin{algorithm}[tb]
\begin{algorithmic}[1]
   \STATE{\bfseries Input:} $\cG=\left(\cV, \cE \right)$, $\cC$, oracle $\oracle$, target feature matrix $X \in \mathbb{R}^{d \times n}$, algorithm parameters $c, \lambda, \sigma>0$
   \label{al:ip}
   
   \STATE Initialize $\Sigma_{\so, 0} \leftarrow \lambda I_{d}$, $\vecb_{\so, 0} \leftarrow \mathbf{0}$, $\hat{\vecw}_{\so,0} \leftarrow \mathbf{0}$ for all $u \in \cV$,
   and UCB $\pucb_{\so,\tar} \leftarrow 1$  for all $\so,\tar \in \cV$ \label{al:ini}

   \FOR{$t=1$ {\bfseries to} $T$}

   \STATE {\label{al:oracle}} Choose $\cS_{t} \leftarrow \oracle \left( \cG, \cC, \pucb \right)$  
      
   \FOR{$\so \in \cS_{t}$ }

	   \STATE Get pairwise influence feedback $\vecy_{\so,t}$ \label{al:get-y}

	    \STATE $\vecb_{\so, t} \leftarrow \vecb_{\so, t-1} + X \vecy_{\so,t}$ \label{al:update-b}
    	 \STATE $\Sigma_{\so, t} \leftarrow \Sigma_{\so, t-1} + \sigma^{-2} X X^{T}$ \label{al:update-gram}
   
   	   \STATE $ \hat{\vecw}_{\so, t} \leftarrow \sigma^{-2} \Sigma_{\so, t}^{-1} \vecb_{\so, t}$ \label{al:update-w}

	   \STATE $\pucb_{\so,\tar} \leftarrow \Proj_{[0,1]} \left[ \langle \hat{\vecw}_{\so, t} \vecx_{\tar} \rangle  + c \| \vecx_{\tar} \|_{\Sigma_{\so, t}^{-1}} \right]$,
	 $\forall \tar \in \cV$
   \ENDFOR 
   
	\FOR{$\so \not \in \cS_{t}$ }
			\STATE $\vecb_{\so, t} = \vecb_{\so, t-1}$
			\STATE $\Sigma_{\so, t} = \Sigma_{\so, t-1}$
	\ENDFOR

   \ENDFOR
\end{algorithmic}
\caption{Diffusion-Independent LinUCB ($\dilinucb$)}
\label{algo:ucb}
\end{algorithm}
The inputs to $\dilinucb$ include the network topology $\cG$, the collection of the feasible sets $\cC$, the optimization algorithm $\oracle$, the target feature matrix $X$, and three algorithm parameters $c, \lambda, \sigma>0$. The parameter $\lambda$ is a 
regularization parameter whereas $\sigma$ is proportional to the noise in the observations and hence controls the learning rate. For each source node $\so \in \cV$ and time $t$, we define the Gram matrix $\Sigma_{\so,t} \in \Re^{d \times d}$, and $\vecb_{\so,t} \in \Re^d$ as the vector summarizing the past propagations from $\so$. The vector $\vecw_{\so,t}$ is the source weight estimate for node $\so$ at round $t$. The mean reachability probability from $\so$ to $\tar$ is given by $\langle \hat{\vecw}_{\so,t} , \vecx_{\tar} \rangle$, whereas its variance is given as $\| \vecx_{\tar} \|_{\Sigma_{\so,t}^{-1}} = \sqrt{\vecx_{\tar}^T \Sigma_{\so, t}^{-1} \vecx_{\tar}}$. Note that $\Sigma_{\so}$ and $\vecb_{\so}$ are sufficient statistics for computing UCB estimates $\pucb_{\so,\tar}$ for all $\tar \in \cV$. The parameter $c$ trades off the mean and variance in the UCB estimates and thus controls the ``degree of optimism" of the algorithm. 

All the Gram matrices are initialized to $\lambda I_{d}$, where $I_{d}$ denotes the $d$-dimensional identity matrix whereas the vectors $\vecb_{\so,0}$ and $\vecw_{\so,0}$ are set to $d$-dimensional all-zeros vectors. At each round $t$, $\dilinucb$ first uses the existing UCB estimates to compute the seed set $\cS_{t}$ based on the given oracle $\oracle$ (line $4$ of \cref{algo:ucb}). Then, it observes the pairwise reachability vector $\vecy_{\so,t}$ for all the selected seeds in $\cS_{t}$. The vector $\vecy_{\so,t}$ is an $n$-dimensional column vector such that $\vecy_{\so,t}(\tar) = \mathds{1} \left(\{\so\}, \tar, \cD (\bw_t) \right)$ indicating whether node $\tar$ is reachable from the source $\so$ at round $t$. Finally, for each of the $K$ selected seeds $\so \in \cS_t$, $\dilinucb$ updates the sufficient statistics  (lines $7$ and $8$ of \cref{algo:ucb}) and the UCB estimates (line $10$ of \cref{algo:ucb}). Here, $\Proj_{[0,1]}[\cdot]$ projects a real number onto the $[0,1]$ interval.
\section{Regret Bound}
\label{sec:bound}
In this section, we derive a regret bound for $\dilinucb$, under (1) Assumption~\ref{assum:monotone}, (2) perfect linear generalization i.e. $p^*_{\so,\tar} = \langle \vecw^{*}_{\so}, \vecx_{\tar} \rangle $ for all $\so, \tar \in \cV$, and (3) the assumption that 
$\vert \vert \vecx_{\tar} \vert \vert_{2} \leq 1$ for all $\tar \in \cV$. Notice that (2) is the standard assumption for linear bandit analysis \cite{dani2008stochastic}, and (3)
can always be satisfied by rescaling the target features.
Our regret bound is stated below:

\begin{theorem}
For any $\lambda, \sigma>0$, any feature matrix $X$, any $\alpha$-approximation oracle $\oracle$, and any $c$ satisfying
\begin{small}
\begin{flalign}
c \geq  \frac{1}{\sigma} \sqrt{dn \log \left( 1+ \frac{nT}{\sigma^2 \lambda d} \right)+ 2 \log \left(n^2 T \right)} + \sqrt{\lambda}   \max_{u \in \cV} \| \vecw^*_\so \|_2,
\label{ineq:c}
\end{flalign}
\end{small}
if we apply $\dilinucb$ with input $( \oracle, X, c, \lambda, \sigma )$, then its
$\rho \alpha$-scaled cumulative regret is upper-bounded as
\begin{small}
\begin{align}
R^{\rho \alpha}(T) \leq & \, \frac{2c}{\rho \alpha} n^{\frac{3}{2}} \sqrt{\frac{d K T \log \left ( 1 + \frac{n T}{d \lambda \sigma^2} \right )}{ \lambda \log \left(1+\frac{1}{\lambda \sigma^2} \right)}}+ \frac{1}{\rho}. 
\end{align}
\end{small}
\vspace{-0.15in}

For the tabular case $X=I$, we obtain a tighter bound
\vspace{-0.15in}
\begin{small}
\begin{align}
R^{\rho \alpha}(T) \leq & \, \frac{2c}{\rho \alpha} n^{\frac{3}{2}} \sqrt{\frac{ K T \log \left ( 1 + \frac{T}{\lambda \sigma^2} \right )}{ \lambda \log \left(1+\frac{1}{\lambda \sigma^2} \right)}}+ \frac{1}{\rho} .
\end{align}
\end{small}
\vspace{-0.15in}
\label{thm:main}
\end{theorem}
Recall that $\rho$ specifies the quality of the surrogate approximation. Notice that if we choose $\lambda=\sigma=1$, and choose $c$ s.t. Inequality~\ref{ineq:c} is tight, then our regret bound is $\tilde{O}(n^2 d \sqrt{KT}/(\alpha \rho))$ for general feature matrix $X$, and
$\tilde{O}(n^{2.5} \sqrt{KT}/(\alpha \rho) )$ in the tabular case. Here the $\tilde{O}$ hides log factors. 
We now briefly discuss the tightness of our regret bounds. First, note that the $O(1/\rho)$ factor is due to the
surrogate objective approximation discussed in \cref{sec:approximation}, and the $O(1/\alpha)$ factor is due to the fact that $\oracle$ is an $\alpha$-approximation algorithm. Second, note that the 
$\tilde{O}(\sqrt{T})$-dependence on time is near-optimal, and the $\tilde{O}(\sqrt{K})$-dependence on the cardinality of the seed sets is standard in the combinatorial semi-bandit literature \cite{kveton2015tight}.
Third, for general $X$, notice that the $\tilde{O}(d)$-dependence on feature dimension is standard in linear bandit literature
\cite{dani2008stochastic, wen2015efficient}. To explain the $\tilde{O}(n^2)$ factor in this case, 
notice that one $O(n)$ factor is due to the magnitude of the reward (the reward is from $0$ to $n$, rather than $0$ to $1$),
whereas one $\tilde{O}(\sqrt{n})$ factor is due to the statistical dependence of the pairwise reachabilities.
Assuming statistical independence between these reachabilities (similar to~\citet{chen2016combinatorial}), we can shave off this $\tilde{O}(\sqrt{n})$ factor. However this assumption is unrealistic in practice. Another $\tilde{O}(\sqrt{n})$ is due to the fact that we learn one $\vecw^*_\so$ for each source node $\so$ (i.e. there is no generalization across the source nodes). Finally, for the tabular case $X=I$, the dependence on $d$ no longer exists, but there is another $\tilde{O}(\sqrt{n})$ factor due to the fact that there is no generalization across target nodes.

We conclude this section by sketching the proof for \cref{thm:main} (the detailed proof is available in 
\cref{sec:proof_for_main} and \cref{sec:matrix_conc}). We define the ``good event" as 
\begin{small}
\[
\cF=\{
| \vecx_{\tar}^T (\hat{\vecw}_{\so, t-1} - \vecw_{\so}^*)| \leq c \| \vecx_{\tar}\|_{\Sigma_{\so, t-1}^{-1}}
 \, \forall \so, \tar  \in \cV, \, t \leq T
 \},
\]
\end{small}
and the ``bad event" $\bar{\cF}$ as the complement of $\cF$.
We then decompose the $\rho \alpha$-scaled regret
$R^{\rho \alpha}(T)$ over $\cF$ and $\bar{\cF}$, and obtain the following inequality:
\begin{small}
\[
R^{\rho \alpha} (T) \leq
\frac{2c}{\rho \alpha} \E \left \{\sum_{t=1}^T \sum_{u \in \cS_t}  \sum_{\tar \in \cV}  \| \vecx_{\tar}\|_{\Sigma_{\so, t-1}^{-1}} \middle | \cF \right \} +
 \frac{P(\bar{\cF})}{\rho} nT,
\]
\end{small}
where $P(\bar{\cF})$ is the probability of $\bar{\cF}$. The regret bounds in \cref{thm:main}
are derived based on worst-case bounds on $\sum_{t=1}^T \sum_{u \in \cS_t}  \sum_{v \in \cV}  \| x_{\tar}\|_{\Sigma_{\so, t-1}^{-1}}$ (\cref{sec:worst-case}), and a bound on $P(\bar{\cF})$ based on the ``self-normalized bound for matrix-valued martingales" developed in~\cref{thm:self-normalization} (\cref{sec:matrix_conc}).


\section{Practical Implementation}
\label{sec:implementation}
In this section, we briefly discuss how to implement our proposed algorithm, $\dilinucb$, in practical semi-bandit IM problems.
Specifically, we will discuss how to construct features in \cref{subsec:features}, how to enhance the practical performance of
$\dilinucb$ based on Laplacian regularization in \cref{subsec:laplacian}, and how to implement $\dilinucb$ computationally efficiently
in real-world problems in \cref{subsec:scalability}.
\subsection{Target Feature Construction}
\label{subsec:features}
Although $\dilinucb$ is applicable with any target feature matrix $X$, in practice, its performance is highly dependent on the ``quality" of $X$. In this subsection, we motivate and propose a systematic feature construction approach based on the unweighted Laplacian matrix of the network topology $\cG$. For all $\so \in \cV$,  let $p^*_\so \in \Re^n$ be the vector encoding the reachabilities from the seed $\so$ to all the target nodes $\tar \in \cV$. Intuitively, $p^*_\so$ tends to be a smooth graph function in the sense that target nodes close to each other (e.g., in the same community) tend to have similar reachabilities from $\so$. From~\cite{belkin2006manifold,valko2014spectral}, we know that a smooth graph function (in this case, the reachability from a source) can be expressed as a linear combination of eigenvectors of the weighted Laplacian of the network. In our case, the edge weights correspond to influence probabilities and are unknown in the IM semi-bandit setting. However, we use the above intuition to construct target features based on the unweighted Laplacian of $\cG$.   Specifically, for a given $d=1,2,\ldots,n$, we set the feature matrix $X$ to be the bottom $d$ eigenvectors (associated with $d$ smallest eigenvalues) of the unweighted Laplacian of $\cG$. Other approaches to construct target features include the neighbourhood preserving node-level features as described in~\cite{grover2016node2vec,perozzi2014deepwalk}. We leave the investigation of other feature construction approaches to future work.

\subsection{Laplacian Regularization}
\label{subsec:laplacian}
One limitation of our proposed $\dilinucb$ algorithm is that it does not generalize across the seed nodes $\so$. Specifically, it needs to learn the source node feature $\vecw^*_\so$ for each source node $\so$ separately, which is inefficient for large-scale semi-bandit IM problems. Similar to target features, the source features also tend to be smooth in the sense that $\| \vecw^*_{\so_1} - \vecw^*_{\so_2} \|_2$ is ``small" if nodes $\so_1$ and $\so_2$ are adjacent. We use this idea to design a prior which ties together the source features for different nodes, and hence transfers information between them. This idea of Laplacian regularization has been used in  multi-task learning~\cite{evgeniou2005learning} and for contextual-bandits in~\cite{cesa2013gang,vaswani2017horde}. Specifically, at each round $t$, we compute $ \hat{\vecw}_{\so,t}$ by minimizing the following objective w.r.t $\vecw_{\so}$: 
\begin{flalign*}
\sum_{j = 1}^{t} \sum_{\so \in \cS_{t}}(\mathbf{y}_{\so,j} -  X^T {\vecw}_{\so})^{2} + \lambda_{2} \sum_{(\so_{1},\so_{2}) \in \cE} \vert \vert {\vecw}_{\so_{1}} - {\vecw}_{\so_{2}} \vert \vert_{2}^{2}
\end{flalign*}
where $\lambda_2 \geq 0$ is the regularization parameter. The implementation details are provided in \cref{sec:lap-reg}. 

\subsection{Computational Complexity}
\label{subsec:scalability}
We now characterize the computational complexity of $\dilinucb$, and discuss how to implement it efficiently. Note that at each time $t$, $\dilinucb$ needs to first compute a solution $\cS_t$ based on $\oracle$, and then update the UCBs. Since $\Sigma_{\so,t}$ is positive semi-definite, the linear system in line $9$ of \cref{algo:ucb} can be solved using conjugate gradient in $O(d^{2})$ time. It is straightforward to see the computational complexity to update the UCBs is $O(Knd^2)$. The computational complexity to compute $\cS_t$ is dependent on $\oracle$. For the classical setting in which $\cC=\{\cS \subseteq \cV :\, |\cS| \leq K\}$ and $\oracle$ is the greedy algorithm, the computational complexity is $O(Kn)$. To speed this up, we use the idea of lazy evaluations for submodular maximization proposed in~\cite{minoux1978accelerated,leskovec2007cost}. It is known that this results in improved running time in practice. 

\newcommand{\cucb}{{\tt CUCB}}
\section{Experiments}
\label{sec:experiments}
\subsection{Empirical Verification of Surrogate Objective}
\label{sec:empirical_objective}
\begin{figure}[!ht]
\centering
        \subfigure[]
        {
			\includegraphics[width=0.2\textwidth,height=3cm]{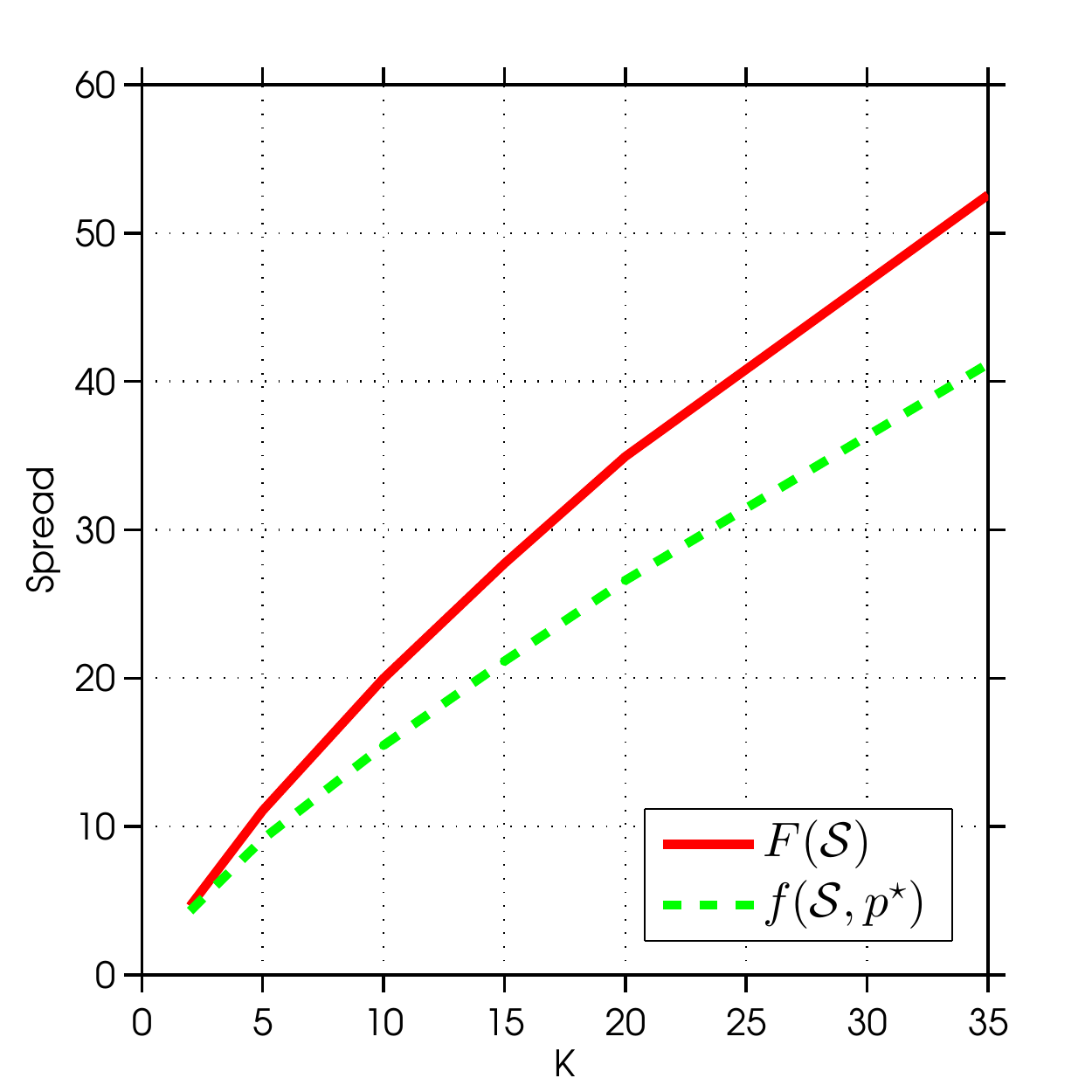}
			\label{fig:ass-1}
        }          
        \subfigure[]
        {
			\includegraphics[width=0.2\textwidth,height=3cm]{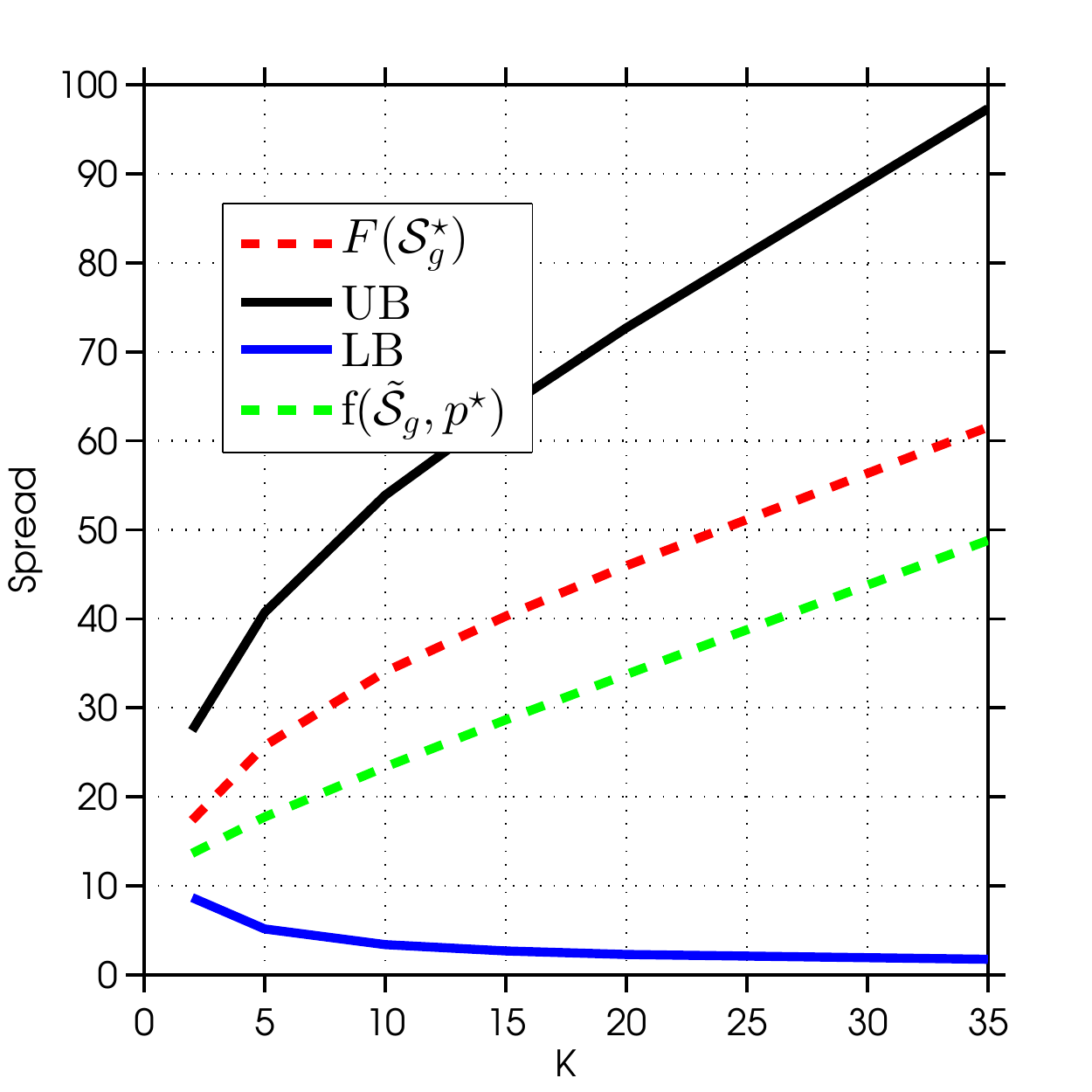}
			\label{fig:ass-2}
        }          
        \caption{Experimental verification of surrogate objective.}      	
\end{figure} 
In this subsection, we empirically verify that the surrogate $f(\cS, p^*)$ proposed in \cref{sec:approximation} is a good approximation of the true IM objective $\spr(\cS)$. We conduct our tests on 
random Kronecker graphs, which are known to capture many properties of real-world social networks~\cite{leskovec2010kronecker}. Specifically, we  generate a \emph{social network instance} $(\cG, \cD)$ as follows: we randomly sample $\cG$ as a Kronecker graph with $n=256$ and \emph{sparsity} equal to $0.03$ \footnote{Based on the sparsity of typical social networks.} \cite{leskovec2005realistic}. We choose $\cD$ as the IC model and sample each of its influence probabilities independently from the uniform distribution $U(0,0.1)$. Note that this range of influence probabilities is guided by the empirical evidence in~\cite{goyal2010learning,barbieri2013topic}. To weaken the dependence on a particular instance, all the results in this subsection are averaged over $10$ randomly generated instances.

We first numerically estimate the pairwise reachabilities $p^*$ for all $10$ instances based on social network simulation.
In a simulation, we randomly sample a seed set $\cS$ with cardinality $K$ between $1$ and $35$, and record the pairwise influence indicator $\vecy_{\so}(\tar)$ from each source $\so \in \cS$ to each target node $\tar$ in this simulation.
The reachability $p^*_{\so, \tar}$ is estimated by averaging the $\vecy_{\so}(\tar)$ values across $50$k such simulations.

Based on the $p^*$ so estimated, we compare $f(\cS, p^*)$ and $\spr(\cS)$ as $K$, the seed set cardinality, varies from $2$ to $35$.
For each $K$ and each social network instance, we randomly sample $100$ seed sets $\cS$ with cardinality $K$.
Then, we evaluate $f(\cS, p^*)$ based on the estimated $p^*$; and numerically evaluate $\spr(\cS)$ by averaging results of
$500$ influence simulations (diffusions).  For each $K$, we average both $\spr(\cS)$ and $f(\cS, p^*)$ across the random seed sets in each instance as well as across the $10$ instances. We plot the average $\spr(\cS)$ and $f(\cS, p^*)$ as a function of $K$ in~\cref{fig:ass-1}.
The plot shows that $f(\cS)$ is a good lower bound on the true expected spread $\spr(\cS)$, especially for low $K$. 

Finally, we empirically quantify the surrogate approximation factor $\rho$. As before, we vary $K$ from $2$ to $35$ and average across
$10$ instances. Let $\alpha^*=1-e^{-1}$.
For each instance and each $K$, we first use the estimated $p^*$ and the greedy algorithm to find an $\alpha^*$-approximation 
solution $\tcS_g$ to the surrogate problem $\max_{\cS} f(\cS, p^*)$. We then use the state-of-the-art IM algorithm~\cite{Tang2014Influence} to compute an $\alpha^*$-approximation 
solution $\cS^*_g$ to the IM problem $\max_{\cS} \spr(\cS)$.
Since $\spr(\cS^*_g) \geq \alpha \spr(\Sopt)$~\cite{nemhauser1978analysis}, $\mathrm{UB}\stackrel{\Delta}{=}\spr(\cS^*_g)/\alpha^*$ is an upper bound on $\spr(\Sopt)$. From \cref{thm:mpr_appr}, $\mathrm{LB}\stackrel{\Delta}{=}\spr(\cS^*_g)/K \leq \spr(\Sopt)/K$ is a lower bound on $f(\tcS, p^*)$. We plot the average values (over $10$ instances) of $\spr(\cS^*_g)$, $f(\tcS_g, p^*)$,  $\mathrm{UB}$
and $\mathrm{LB}$ against $K$ in \cref{fig:ass-2}. We observe that the difference in spreads does not increase rapidly with $K$. Although $\rho$ is lower-bounded with $\frac{1}{K}$, in practice for all $K \in [2,35]$, $\rho \geq \frac{\alpha^* f(\tcS_g, p^*) }{\spr(\cS^*_g)} \geq 0.55$. This shows that in practice, our surrogate approximation is reasonable even for large $K$. 

\subsection{Performance of $\dilinucb$}
\label{sec:performance}
\begin{figure*}[!ht]
\centering
        \subfigure[IC Model]
        {
			\includegraphics[width=0.3\textwidth,height=4cm]{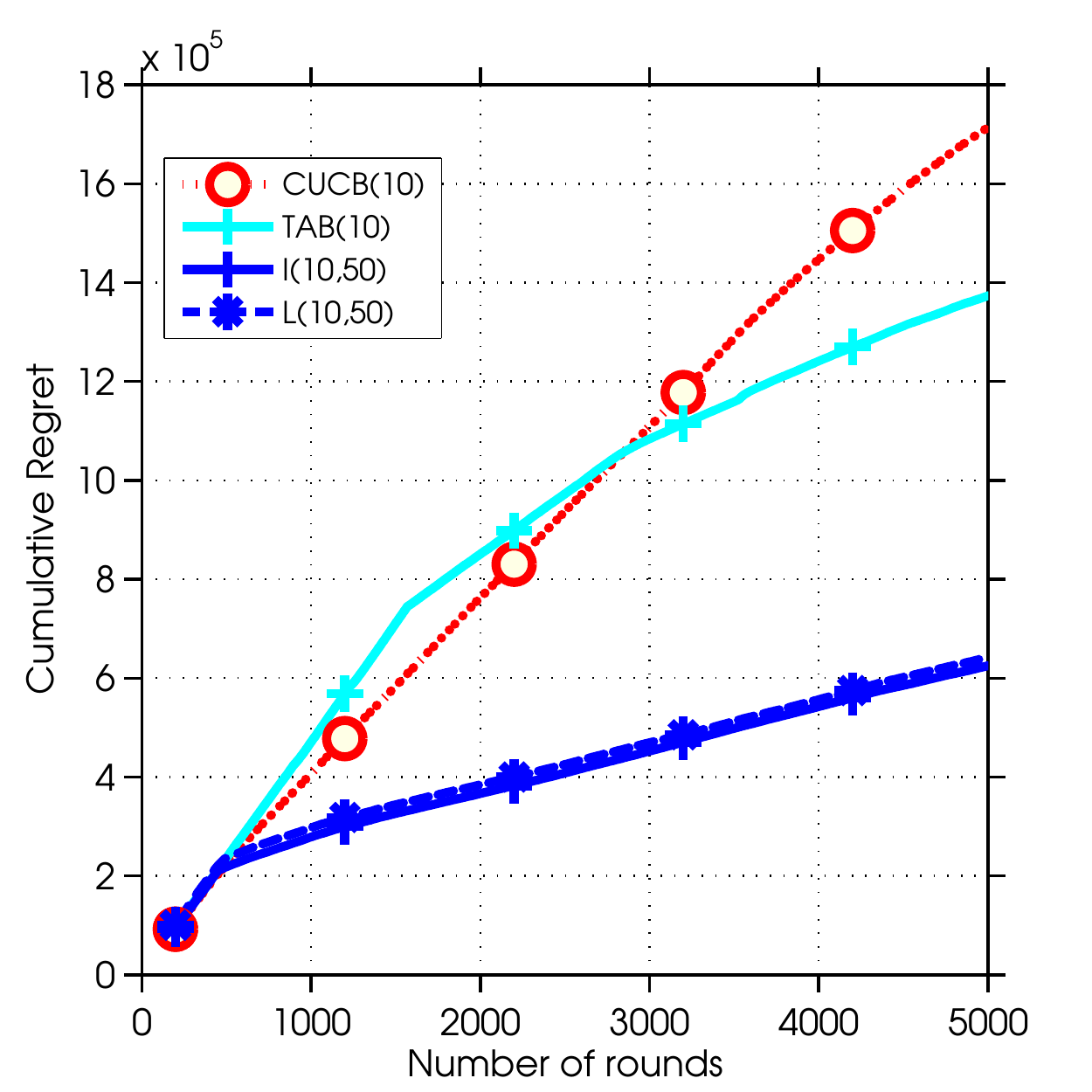}
			\label{fig:IC-facebook-cumulative-regret}
        }          
        \subfigure[LT Model]
        {
			\includegraphics[width=0.3\textwidth,height=4cm]{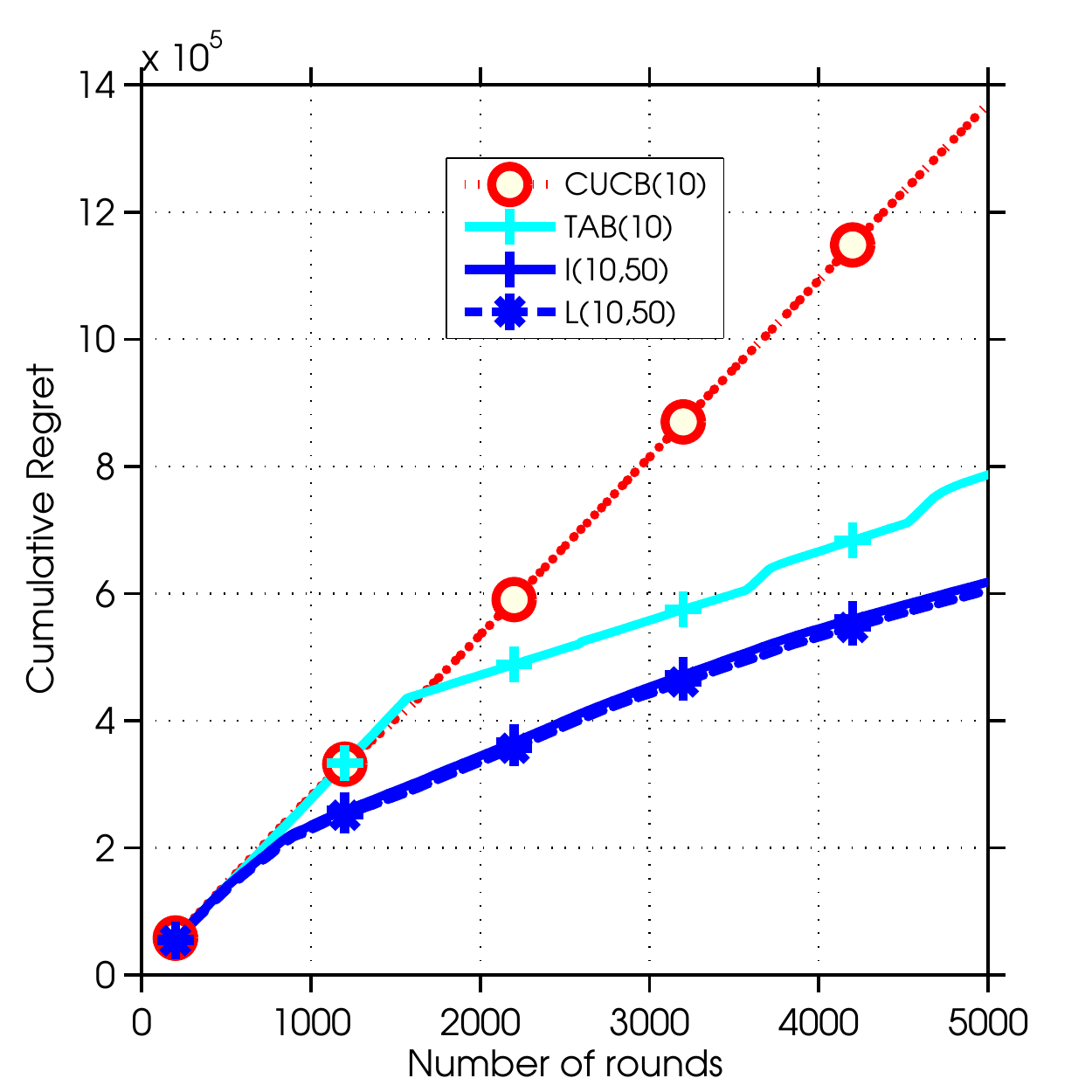}
			\label{fig:LT-facebook-cumulative-regret}
        }          
        \caption{Comparing $\dilinucb$ and $\cucb$ on the Facebook subgraph with $K=10$.}      	
\end{figure*}    

We now demonstrate the performances of variants of $\dilinucb$ and compare them with the start of the art.
We choose the social network topology $\cG$ as a subgraph of the Facebook network available at~\cite{snapnets},
which consists of $n = 4$k nodes and $m = 88$k edges. Since true diffusion model is unavailable, we assume the
diffusion model $\cD$ is either an IC model or an LT model, and sample the edge influence probabilities independently from the
uniform distribution $U(0,0.1)$. We also choose $T = 5$k rounds.

We compare $\dilinucb$ against the $\cucb$ algorithm~\cite{chen2016combinatorial} in both the IC model and the LT model, with $K=10$.
$\cucb$ (referred to as CUCB($K$) in plots) assumes the IC model, edge-level feedback and learns the influence probability for each edge independently.  We demonstrate the performance of three variants of $\dilinucb$ - the tabular case with $X=I$, independent estimation for each source node using target features (\cref{algo:ucb}) and Laplacian regularized estimation with target features (\cref{sec:lap-reg}). In the subsequent plots, to emphasize the dependence on $K$ and $d$, these are referred to as TAB($K$), I($K$,$d$) and L($K$,$d$) respectively. We construct features as described in \cref{subsec:features}. Similar to spectral clustering~\cite{von2007tutorial}, the gap in the eigenvalues of the unweighted Laplacian can be used to choose the number of eigenvectors $d$. In our case, we choose the bottom $d = 50$ eigenvectors for constructing target features and show the effect of varying $d$ in the next experiment. Similar to~\cite{gentile2014online}, all hyper-parameters for our algorithm are set using an initial validation set of $500$ rounds. The best validation performance was observed for  $\lambda = 10^{-4}$ and $\sigma = 1$. 

We now briefly discuss the performance metrics used in this section. For all $\cS \subseteq \cV$ and all $t=1,2\ldots$, we define
$r_t(\cS)=\sum_{\tar \in \cV} I \left( \cS, \tar, \cD(\bw_t) \right) $, which is the realized reward at time $t$ if $\cS$ is chosen at that time. One performance metric is the \emph{per-step reward}. Specifically, in one simulation, the per-step reward at time $t$ is defined as 
$\frac{\sum_{s = 1}^{t} r_{s}}{t}$. Another performance metric is the \emph{cumulative regret}.
Since it is computationally intractable to derive $\Sopt$, our regret is measured with respect to $\cS^*_g$, the $\alpha^*$-approximation
solution discussed in \cref{sec:empirical_objective}. In one simulation, the cumulative regret at time $t$ is defined as
$R(t)=\sum_{s=1}^t \left[ r_s(\cS^*_g)- r_s(\cS_s) \right]$. All the subsequent results are averaged across $5$ independent simulations.


\begin{figure*}[!ht]
\centering
        \subfigure[Effect of $d$ in IC]
        {
			\includegraphics[width=0.3\textwidth,height=4cm]{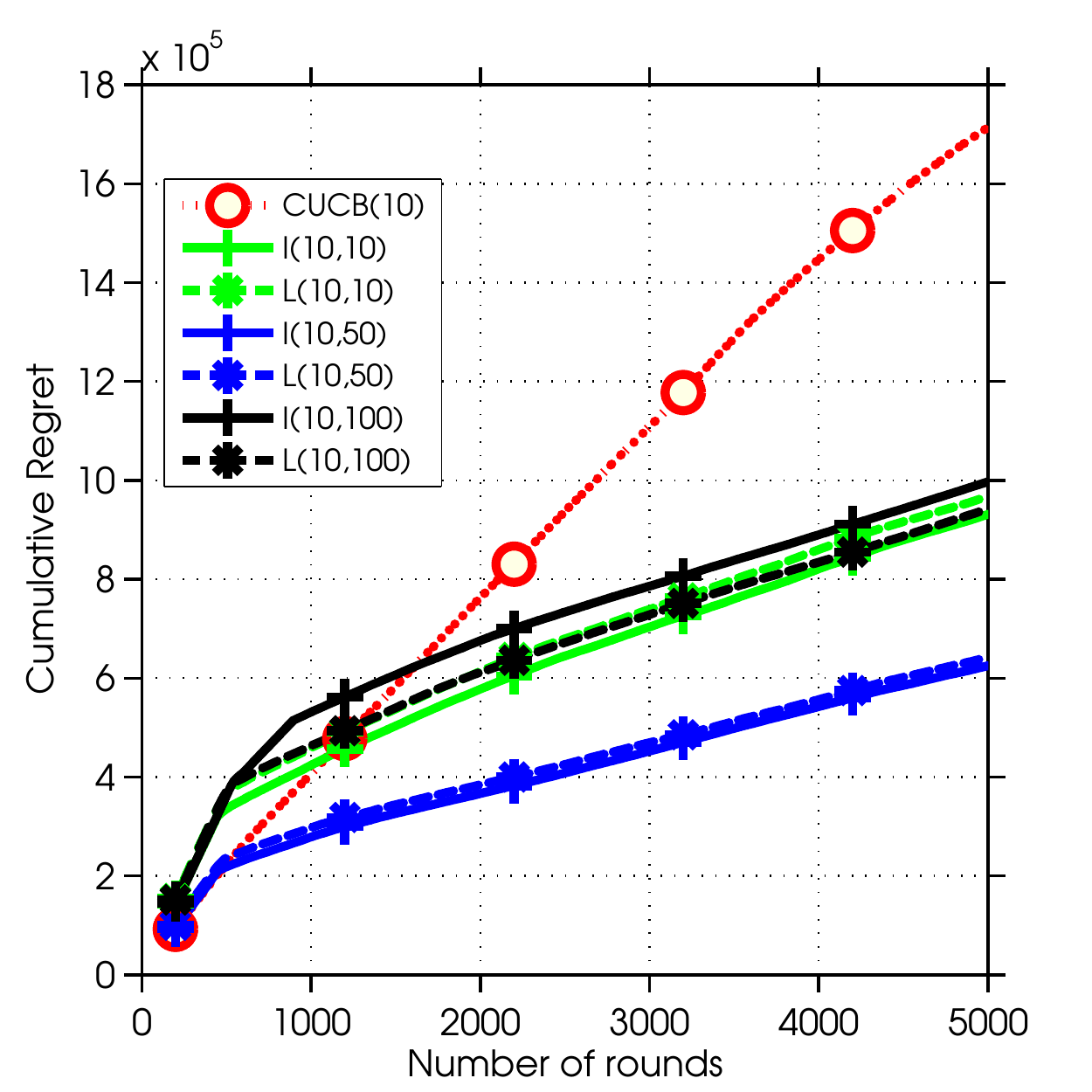}
			\label{fig:IC-facebook-d-variation}
        }          
        \subfigure[Effect of $K$ in IC]
        {
			\includegraphics[width=0.3\textwidth,height=4cm]{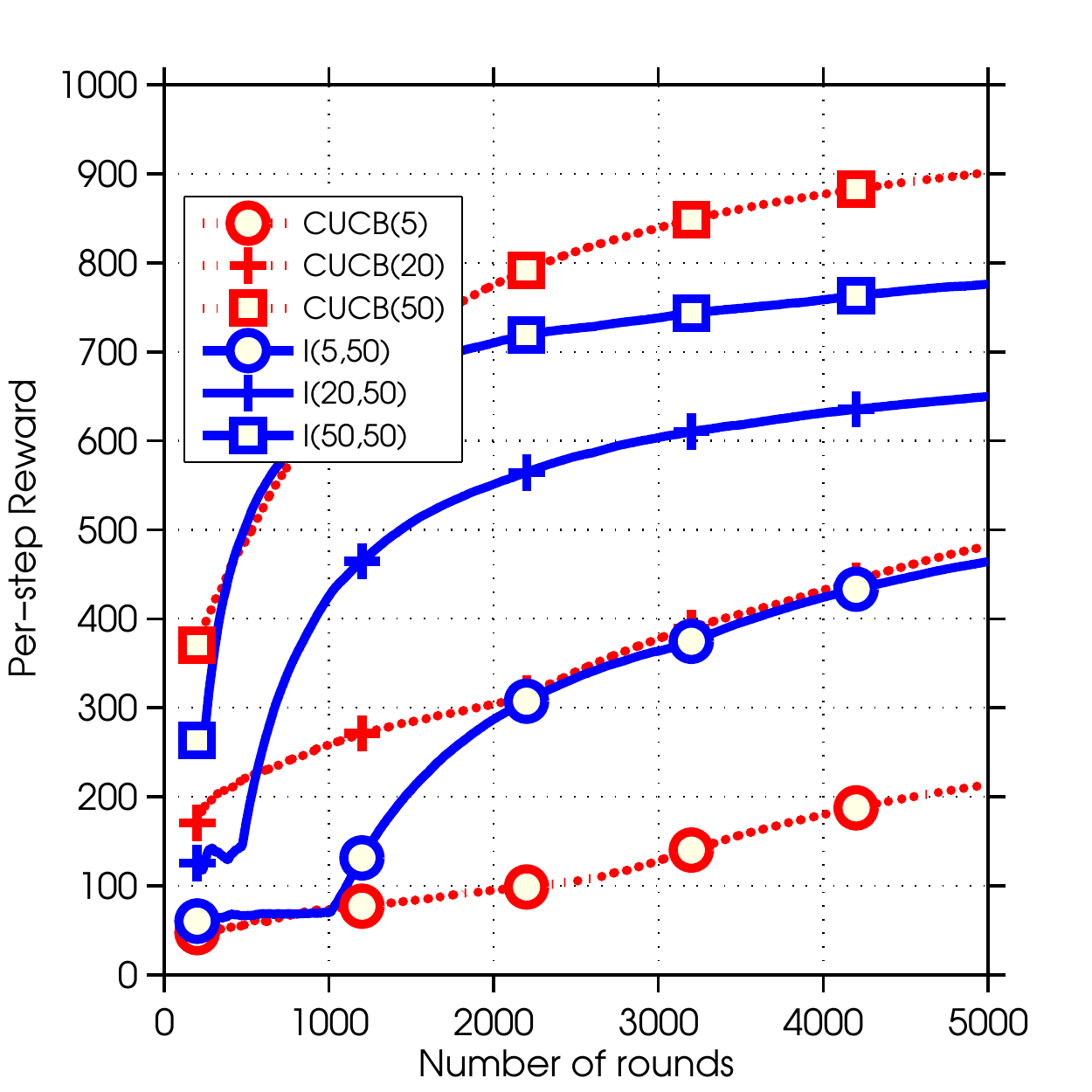}
			\label{fig:IC-facebook-k-variation}
        }          
        \subfigure[Effect of $K$ in LT]
        {
			\includegraphics[width=0.3\textwidth,height=4cm]{facebook-combined-reward-k-variation.pdf}
			\label{fig:LT-facebook-k-variation}
        }          
        
        \caption{Effects of varying $d$ or $K$.}      	
\end{figure*}    
Figures~\ref{fig:IC-facebook-cumulative-regret} and~\ref{fig:LT-facebook-cumulative-regret} show the cumulative regret when the underlying diffusion model is IC and LT, respectively. We have the following observations: (i) As compared to $\cucb$, the cumulative regret increases at a slower rate for all variants of $\dilinucb$, under both the IC and LT models, and for both the tabular case and case with features. (ii) Exploiting target features (linear generalization) in $\dilinucb$ leads to a much smaller cumulative regret. (iii) $\cucb$ is not robust to model misspecification: it has a near linear cumulative regret under LT model. (iv) Laplacian regularization has little effect on the cumulative regret in these two cases. These observations clearly demonstrate the two main advantages of $\dilinucb$: it is both statistically efficient and robust to diffusion model misspecification. To explain (iv), we argue that the current combination of $T$, $K$, $d$ and $n$ results in sufficient feedback for independent estimation to perform well and hence it is difficult to observe any additional benefit of Laplacian regularization. We provide additional evidence for this argument in the next experiment. 

In \cref{fig:IC-facebook-d-variation}, we quantify the effect of varying $d$ when the underlying diffusion model is IC and make the following observations: (i) The cumulative regret for both $d = 10$ and $d = 100$ is higher than that for $d = 50$. (ii) Laplacian regularization leads to observably lower cumulative regret when $d = 100$. Observation (iii) implies that $d = 10$ does not provide enough expressive power for linear generalization across the nodes of the network, whereas it is relatively difficult to estimate $100$-dimensional $\vecw^*_{\so}$ vectors within $5$k rounds. Observation (iv) implies that tying source node estimates together imposes an additional bias which becomes important while learning higher dimensional coefficients. This shows the potential benefit of using Laplacian regularization for larger networks, where we will need higher $d$ for linear generalization across nodes. We obtain similar results under the LT model. 

In Figures~\ref{fig:IC-facebook-k-variation} and~\ref{fig:LT-facebook-k-variation}, we show the effect of varying $K$ on the per-step reward. We compare $\cucb$ and the independent version of our algorithm when the underlying model is IC and LT. We make the following observations: (i) For both IC and LT, the per-step reward for all methods increases with $K$. (ii) For the IC model, the per-step reward for our algorithm is higher than $\cucb$ when $K = \{5,10,20\}$, but the difference in the two spreads decreases with $K$. For $K = 50$, $\cucb$ outperforms our algorithm. (iii) For the LT model, the per-step reward of our algorithm is substantially higher than $\cucb$ for all $K$. Observation (i) is readily explained since both IC and LT are progressive models, and satisfy Assumption~\ref{assum:monotone}. To explain (ii), note that $\cucb$ is correctly specified for the IC model. As $K$ becomes higher, more edges become active and $\cucb$ observes more feedback. It is thus able to learn more efficiently, leading to a higher per-step reward compared to our algorithm when $K = 50$. Observation (iii) again demonstrates that $\cucb$ is not robust to diffusion model misspecification, while $\dilinucb$ is.

\section{Related Work}
\label{sec:related-work}
IM semi-bandits have been studied in several recent papers~\cite{wen2016influence, chen2016combinatorial,vaswani2015influence,carpentier2016revealing}. \citet{chen2016combinatorial} studied IM semi-bandit under edge-level feedback and the IC diffusion model. They formulated it as a combinatorial multi-armed bandit problem and proposed a UCB algorithm ($\cucb$). They only consider the tabular case, and derive an $O(n^3)$ regret bound that also depends on the reciprocal of the minimum observation probability $p$ of an edge. This can be problematic in for example, a line graph with $L$ edges where all edge weights are $0.5$. Then $1/p$ is $2^{L - 1}$, implying an exponentially large regret. Moreover, they assume that source nodes influence the target nodes independently, which is not true in most practical social networks.
In contrast, both our algorithm and analysis are diffusion independent, and our analysis does not require the ``independent influence" assumption made in \cite{chen2016combinatorial}. Our regret bound is $O(n^{2.5})$ in the tabular case and $O(n^{2}d)$ in the general linear bandit case. \citet{vaswani2015influence} use $\epsilon$-greedy and Thompson sampling algorithms for a different and more challenging feedback model, where the learning agent observes influenced nodes but not the edges. They do not give any theoretical guarantees. Concurrent to our work, \citet{wen2016influence} consider a linear generalization model across edges and prove regret bounds under edge-level feedback. Note that all of the above papers assume the IC diffusion model.

\citet{carpentier2016revealing,fang2014networked} consider a simpler local model of influence, in which information does not transitively diffuse across the network. \citet{lei2015online} consider the related, but different, problem of maximizing the number of unique activated nodes across multiple rounds. They do not provide any theoretical analysis.

\vspace*{-2ex}
\section{Conclusion}
\label{sec:discussion}
In this paper, we described a novel model-independent parametrization and a corresponding surrogate objective function for the IM problem. We used this parametrization to propose $\dilinucb$, a diffusion-independent learning algorithm for IM semi-bandits. We conjecture that with an appropriate generalization across source nodes, it may be possible to get a more statistically efficient algorithm and get rid of an additional $O(\sqrt{n})$ factor in the regret bound. In the future, we hope to address alternate bandit algorithms such Thompson sampling, and feedback models such as node-level in~\citet{vaswani2015influence}.

\textbf{Acknowledgements:} This research was supported by the Natural Sciences and Engineering Research Council of Canada.

\clearpage


\begin{thebibliography}{40}
\providecommand{\natexlab}[1]{#1}
\providecommand{\url}[1]{\texttt{#1}}
\expandafter\ifx\csname urlstyle\endcsname\relax
  \providecommand{\doi}[1]{doi: #1}\else
  \providecommand{\doi}{doi: \begingroup \urlstyle{rm}\Url}\fi

\bibitem[Abbasi-Yadkori et~al.(2011)Abbasi-Yadkori, P{\'a}l, and
  Szepesv{\'a}ri]{abbasi2011improved}
Abbasi-Yadkori, Yasin, P{\'a}l, D{\'a}vid, and Szepesv{\'a}ri, Csaba.
\newblock Improved algorithms for linear stochastic bandits.
\newblock In \emph{Advances in Neural Information Processing Systems}, pp.\
  2312--2320, 2011.

\bibitem[Barbieri et~al.(2013)Barbieri, Bonchi, and Manco]{barbieri2013topic}
Barbieri, Nicola, Bonchi, Francesco, and Manco, Giuseppe.
\newblock Topic-aware social influence propagation models.
\newblock \emph{Knowledge and information systems}, 37\penalty0 (3):\penalty0
  555--584, 2013.

\bibitem[Belkin et~al.(2006)Belkin, Niyogi, and Sindhwani]{belkin2006manifold}
Belkin, Mikhail, Niyogi, Partha, and Sindhwani, Vikas.
\newblock Manifold regularization: A geometric framework for learning from
  labeled and unlabeled examples.
\newblock \emph{Journal of machine learning research}, 7\penalty0
  (Nov):\penalty0 2399--2434, 2006.

\bibitem[Carpentier \& Valko(2016)Carpentier and
  Valko]{carpentier2016revealing}
Carpentier, Alexandra and Valko, Michal.
\newblock {Revealing graph bandits for maximizing local influence}.
\newblock In \emph{International Conference on Artificial Intelligence and
  Statistics}, 2016.

\bibitem[Cesa-Bianchi et~al.(2013)Cesa-Bianchi, Gentile, and
  Zappella]{cesa2013gang}
Cesa-Bianchi, Nicolo, Gentile, Claudio, and Zappella, Giovanni.
\newblock A gang of bandits.
\newblock In \emph{Advances in Neural Information Processing Systems}, pp.\
  737--745, 2013.

\bibitem[Chen et~al.(2009)Chen, Wang, and Yang]{chen2009efficient}
Chen, Wei, Wang, Yajun, and Yang, Siyu.
\newblock Efficient influence maximization in social networks.
\newblock In \emph{Proceedings of the 15th ACM SIGKDD international conference
  on Knowledge discovery and data mining}, pp.\  199--208. ACM, 2009.

\bibitem[Chen et~al.(2016)Chen, Wang, Yuan, and Wang]{chen2016combinatorial}
Chen, Wei, Wang, Yajun, Yuan, Yang, and Wang, Qinshi.
\newblock Combinatorial multi-armed bandit and its extension to
  probabilistically triggered arms.
\newblock \emph{Journal of Machine Learning Research}, 17\penalty0
  (50):\penalty0 1--33, 2016.

\bibitem[Dani et~al.(2008)Dani, Hayes, and Kakade]{dani2008stochastic}
Dani, Varsha, Hayes, Thomas~P, and Kakade, Sham~M.
\newblock Stochastic linear optimization under bandit feedback.
\newblock In \emph{COLT}, pp.\  355--366, 2008.

\bibitem[Du et~al.(2014)Du, Liang, Balcan, and Song]{Du2014}
Du, Nan, Liang, Yingyu, Balcan, Maria-Florina, and Song, Le.
\newblock {Influence Function Learning in Information Diffusion Networks}.
\newblock \emph{Journal of Machine Learning Research}, 32:\penalty0 2016--2024,
  2014.
\newblock URL
  \url{http://machinelearning.wustl.edu/mlpapers/papers/icml2014c2{\_}du14}.

\bibitem[Evgeniou et~al.(2005)Evgeniou, Micchelli, and
  Pontil]{evgeniou2005learning}
Evgeniou, Theodoros, Micchelli, Charles~A, and Pontil, Massimiliano.
\newblock Learning multiple tasks with kernel methods.
\newblock \emph{Journal of Machine Learning Research}, 6\penalty0
  (Apr):\penalty0 615--637, 2005.

\bibitem[Fang \& Tao(2014)Fang and Tao]{fang2014networked}
Fang, Meng and Tao, Dacheng.
\newblock {Networked bandits with disjoint linear payoffs}.
\newblock In \emph{Internattional Conference on Knowledge Discovery and Data
  Mining}, 2014.

\bibitem[Gentile et~al.(2014)Gentile, Li, and Zappella]{gentile2014online}
Gentile, Claudio, Li, Shuai, and Zappella, Giovanni.
\newblock Online clustering of bandits.
\newblock In \emph{Proceedings of the 31st International Conference on Machine
  Learning (ICML-14)}, pp.\  757--765, 2014.

\bibitem[Gomez~Rodriguez et~al.(2012)Gomez~Rodriguez, Sch{\"o}lkopf, Pineau,
  et~al.]{gomez2012influence}
Gomez~Rodriguez, M, Sch{\"o}lkopf, B, Pineau, Langford~J, et~al.
\newblock Influence maximization in continuous time diffusion networks.
\newblock In \emph{29th International Conference on Machine Learning (ICML
  2012)}, pp.\  1--8. International Machine Learning Society, 2012.

\bibitem[Goyal et~al.(2010)Goyal, Bonchi, and Lakshmanan]{goyal2010learning}
Goyal, Amit, Bonchi, Francesco, and Lakshmanan, Laks~VS.
\newblock Learning influence probabilities in social networks.
\newblock In \emph{Proceedings of the third ACM international conference on Web
  search and data mining}, pp.\  241--250. ACM, 2010.

\bibitem[Goyal et~al.(2011{\natexlab{a}})Goyal, Bonchi, and
  Lakshmanan]{goyal2011data}
Goyal, Amit, Bonchi, Francesco, and Lakshmanan, Laks~VS.
\newblock A data-based approach to social influence maximization.
\newblock \emph{Proceedings of the VLDB Endowment}, 5\penalty0 (1):\penalty0
  73--84, 2011{\natexlab{a}}.

\bibitem[Goyal et~al.(2011{\natexlab{b}})Goyal, Lu, and
  Lakshmanan]{goyal2011simpath}
Goyal, Amit, Lu, Wei, and Lakshmanan, Laks~VS.
\newblock Simpath: An efficient algorithm for influence maximization under the
  linear threshold model.
\newblock In \emph{Data Mining (ICDM), 2011 IEEE 11th International Conference
  on}, pp.\  211--220. IEEE, 2011{\natexlab{b}}.

\bibitem[Grover \& Leskovec(2016)Grover and Leskovec]{grover2016node2vec}
Grover, Aditya and Leskovec, Jure.
\newblock node2vec: Scalable feature learning for networks.
\newblock In \emph{Proceedings of the 22nd ACM SIGKDD International Conference
  on Knowledge Discovery and Data Mining}, pp.\  855--864. ACM, 2016.

\bibitem[Hestenes \& Stiefel(1952)Hestenes and Stiefel]{hestenes1952methods}
Hestenes, Magnus~Rudolph and Stiefel, Eduard.
\newblock \emph{Methods of conjugate gradients for solving linear systems},
  volume~49.
\newblock 1952.

\bibitem[Kempe et~al.(2003)Kempe, Kleinberg, and Tardos]{kempe2003maximizing}
Kempe, David, Kleinberg, Jon, and Tardos, {\'E}va.
\newblock Maximizing the spread of influence through a social network.
\newblock In \emph{Proceedings of the ninth ACM SIGKDD international conference
  on Knowledge discovery and data mining}, pp.\  137--146. ACM, 2003.

\bibitem[Krause \& Golovin(2012)Krause and Golovin]{krause2012submodular}
Krause, Andreas and Golovin, Daniel.
\newblock Submodular function maximization.
\newblock \emph{Tractability: Practical Approaches to Hard Problems},
  3\penalty0 (19):\penalty0 8, 2012.

\bibitem[Kveton et~al.(2015)Kveton, Wen, Ashkan, and
  Szepesvari]{kveton2015tight}
Kveton, Branislav, Wen, Zheng, Ashkan, Azin, and Szepesvari, Csaba.
\newblock Tight regret bounds for stochastic combinatorial semi-bandits.
\newblock In \emph{AISTATS}, 2015.

\bibitem[Lei et~al.(2015)Lei, Maniu, Mo, Cheng, and Senellart]{lei2015online}
Lei, Siyu, Maniu, Silviu, Mo, Luyi, Cheng, Reynold, and Senellart, Pierre.
\newblock Online influence maximization.
\newblock In \emph{Proceedings of the 21th {ACM} {SIGKDD} International
  Conference on Knowledge Discovery and Data Mining, Sydney, NSW, Australia,
  August 10-13, 2015}, pp.\  645--654, 2015.

\bibitem[Leskovec \& Krevl(2014)Leskovec and Krevl]{snapnets}
Leskovec, Jure and Krevl, Andrej.
\newblock {SNAP Datasets}: {Stanford} large network dataset collection.
\newblock \url{http://snap.stanford.edu/data}, June 2014.

\bibitem[Leskovec et~al.(2007)Leskovec, Krause, Guestrin, Faloutsos,
  VanBriesen, and Glance]{leskovec2007cost}
Leskovec, Jure, Krause, Andreas, Guestrin, Carlos, Faloutsos, Christos,
  VanBriesen, Jeanne, and Glance, Natalie.
\newblock Cost-effective outbreak detection in networks.
\newblock In \emph{Proceedings of the 13th ACM SIGKDD international conference
  on Knowledge discovery and data mining}, pp.\  420--429. ACM, 2007.

\bibitem[Leskovec et~al.(2010)Leskovec, Chakrabarti, Kleinberg, Faloutsos, and
  Ghahramani]{leskovec2010kronecker}
Leskovec, Jure, Chakrabarti, Deepayan, Kleinberg, Jon, Faloutsos, Christos, and
  Ghahramani, Zoubin.
\newblock Kronecker graphs: An approach to modeling networks.
\newblock \emph{The Journal of Machine Learning Research}, 11:\penalty0
  985--1042, 2010.

\bibitem[Leskovec et~al.(2005)Leskovec, Chakrabarti, Kleinberg, and
  Faloutsos]{leskovec2005realistic}
Leskovec, Jurij, Chakrabarti, Deepayan, Kleinberg, Jon, and Faloutsos,
  Christos.
\newblock Realistic, mathematically tractable graph generation and evolution,
  using kronecker multiplication.
\newblock In \emph{European Conference on Principles of Data Mining and
  Knowledge Discovery}, pp.\  133--145. Springer, 2005.

\bibitem[Li et~al.(2013)Li, Chen, Wang, and Zhang]{li2013influence}
Li, Yanhua, Chen, Wei, Wang, Yajun, and Zhang, Zhi-Li.
\newblock Influence diffusion dynamics and influence maximization in social
  networks with friend and foe relationships.
\newblock In \emph{Proceedings of the sixth ACM international conference on Web
  search and data mining}, pp.\  657--666. ACM, 2013.

\bibitem[Minoux(1978)]{minoux1978accelerated}
Minoux, Michel.
\newblock Accelerated greedy algorithms for maximizing submodular set
  functions.
\newblock In \emph{Optimization Techniques}, pp.\  234--243. Springer, 1978.

\bibitem[Nemhauser et~al.(1978)Nemhauser, Wolsey, and
  Fisher]{nemhauser1978analysis}
Nemhauser, George~L, Wolsey, Laurence~A, and Fisher, Marshall~L.
\newblock An analysis of approximations for maximizing submodular set
  functions.
\newblock \emph{Mathematical Programming}, 14\penalty0 (1):\penalty0 265--294,
  1978.

\bibitem[Netrapalli \& Sanghavi(2012)Netrapalli and
  Sanghavi]{netrapalli2012learning}
Netrapalli, Praneeth and Sanghavi, Sujay.
\newblock Learning the graph of epidemic cascades.
\newblock In \emph{ACM SIGMETRICS Performance Evaluation Review}, volume~40,
  pp.\  211--222. ACM, 2012.

\bibitem[Perozzi et~al.(2014)Perozzi, Al-Rfou, and Skiena]{perozzi2014deepwalk}
Perozzi, Bryan, Al-Rfou, Rami, and Skiena, Steven.
\newblock Deepwalk: Online learning of social representations.
\newblock In \emph{Proceedings of the 20th ACM SIGKDD international conference
  on Knowledge discovery and data mining}, pp.\  701--710. ACM, 2014.

\bibitem[Saito et~al.(2008)Saito, Nakano, and Kimura]{saito2008prediction}
Saito, Kazumi, Nakano, Ryohei, and Kimura, Masahiro.
\newblock Prediction of information diffusion probabilities for independent
  cascade model.
\newblock In \emph{Knowledge-Based Intelligent Information and Engineering
  Systems}, pp.\  67--75. Springer, 2008.

\bibitem[Tang et~al.(2014)Tang, Xiao, and Yanchen]{Tang2014Influence}
Tang, Youze, Xiao, Xiaokui, and Yanchen, Shi.
\newblock Influence maximization: Near-optimal time complexity meets practical
  efficiency.
\newblock 2014.

\bibitem[Tang et~al.(2015)Tang, Shi, and Xiao]{Tang2015TIM+}
Tang, Youze, Shi, Yanchen, and Xiao, Xiaokui.
\newblock Influence maximization in near-linear time: A martingale approach.
\newblock In \emph{Proceedings of the 2015 ACM SIGMOD International Conference
  on Management of Data}, SIGMOD '15, pp.\  1539--1554, 2015.
\newblock ISBN 978-1-4503-2758-9.

\bibitem[Valko et~al.(2014)Valko, Munos, Kveton, and
  Koc{\'a}k]{valko2014spectral}
Valko, Michal, Munos, R{\'e}mi, Kveton, Branislav, and Koc{\'a}k,
  Tom{\'a}{\v{s}}.
\newblock Spectral bandits for smooth graph functions.
\newblock In \emph{31th International Conference on Machine Learning}, 2014.

\bibitem[Vaswani et~al.(2015)Vaswani, Lakshmanan, and {Mark
  Schmidt}]{vaswani2015influence}
Vaswani, Sharan, Lakshmanan, Laks. V.~S., and {Mark Schmidt}.
\newblock {Influence maximization with bandits}.
\newblock Technical report, http://arxiv.org/abs/1503.00024, 2015.
\newblock URL \url{http://arxiv.org/abs/1503.00024}.

\bibitem[Vaswani et~al.(2017)Vaswani, Schmidt, and
  Lakshmanan]{vaswani2017horde}
Vaswani, Sharan, Schmidt, Mark, and Lakshmanan, Laks.
\newblock Horde of bandits using gaussian markov random fields.
\newblock In \emph{Artificial Intelligence and Statistics}, pp.\  690--699,
  2017.

\bibitem[Von~Luxburg(2007)]{von2007tutorial}
Von~Luxburg, Ulrike.
\newblock A tutorial on spectral clustering.
\newblock \emph{Statistics and computing}, 17\penalty0 (4):\penalty0 395--416,
  2007.

\bibitem[Wen et~al.(2015)Wen, Kveton, and Ashkan]{wen2015efficient}
Wen, Zheng, Kveton, Branislav, and Ashkan, Azin.
\newblock Efficient learning in large-scale combinatorial semi-bandits.
\newblock In \emph{Proceedings of the 32nd International Conference on Machine
  Learning, {ICML} 2015, Lille, France, 6-11 July 2015}, 2015.

\bibitem[Wen et~al.(2017)Wen, Kveton, Valko, and Vaswani]{wen2016influence}
Wen, Zheng, Kveton, Branislav, Valko, Michal, and Vaswani, Sharan.
\newblock Online influence maximization under independent cascade model with
  semi-bandit feedback.
\newblock \emph{arXiv preprint arXiv:1605.06593v2}, 2017.

\end{thebibliography}

\clearpage

\onecolumn

\appendix
\begin{center}
{\Large \textbf{Appendices}}
\end{center}

\section{Proof of \cref{thm:mpr_appr}}
\label{sec:mpr_appr}

\begin{proof}
\cref{thm:mpr_appr} can be proved based on the definitions of monotonicity and submodularity. Note that from Assumption~\ref{assum:monotone}, for any seed set $\cS \in \cC$, any seed node $u \in \cS$, and any target node $v \in \cV$, we have
$
\spr( \{ u\}, v) \leq \spr(\cS, v)
$,
which implies that
\[
f (\cS, v, p^*) = \max_{u \in \cS} \spr(\{ u\}, v) \leq \spr(\cS, v),
\]
hence
\[
f (\cS, p^*) = \sum_{v \in \cV} f (\cS, v, p^*) \leq \sum_{v \in \cV} \spr(\cS, v) = \spr(\cS).
\]
This proves the first part of \cref{thm:mpr_appr}.

We now prove the second part of the theorem. First, note that from the first part, we have
\[
f(\tilde{\cS}, p^*) \leq \spr(\tilde{\cS}) \leq \spr(\seeds^*),
\]
where the first inequality follows from the first part of this theorem, and the second inequality follows from the definition of
$\seeds^*$. Thus, we have $\rho \leq 1$.
To prove that $\rho \geq 1/K$, we assume that $\cS=\left \{u_1, u_2, \ldots, u_K \right \}$, and define
$\cS_k = \left \{u_1, u_2, \ldots, u_k \right \}$ for $k=1,2,\ldots, K$. 
Thus, for any $\cS \subseteq \cV$ with $|\cS|=K$, we have
\begin{align*}
\spr(\seeds) =&\, \spr(\cS_1) + \sum_{k=1}^{K-1} \left[ \spr(\cS_{k+1}) -  \spr(\cS_k) \right] \\
\leq & \, \sum_{k=1}^K \spr (\{u_k \}) = \sum_{k=1}^K \sum_{v \in \cV} \spr (\{u_k \}, v) \\
\leq & \, \sum_{v \in \cV} K \max_{u \in \seeds}  \spr (\{u \}, v)
= K \sum_{v \in \cV} f(\cS, v, p^*) = K f(\cS, p^*),
\end{align*}
where the first inequality follows from the submodularity of $\spr(\cdot)$. Thus we have
\[
\spr(\seeds^*) \leq K f(\seeds^*, p^*) \leq K f(\tilde{\seeds}, p^*),
\]
where the second inequality follows from the definition of $\tilde{\seeds}$. This implies that $\rho \geq 1/K$.
\end{proof}


\section{Proof of \cref{thm:main}}
\label{sec:proof_for_main}
We start by defining some useful notations. We use $\cH_t$ to denote the ``history" by the end of time $t$. For any node pair $(u,v) \in \cV \times \cV$
and any time $t$, we define the upper confidence bound (UCB) $U_t(u,v)$ and the lower confidence bound (LCB) $L_t(u,v)$ respectively as
\begin{align}
U_t (u,v) =& \,  \Proj_{[0,1]} \bigg( \langle \hat{\vecw}_{\so,t-1}, \vecx_{v} \rangle + c \sqrt{ \vecx_{v}^{T} \Sigma_{\so,t-1}^{-1} \vecx_{v}} \bigg)
\nonumber \\
L_t (u,v) =& \,  \Proj_{[0,1]} \bigg( \langle \hat{\vecw}_{\so,t-1}, \vecx_{v} \rangle - c \sqrt{ \vecx_{v}^{T} \Sigma_{\so,t-1}^{-1} \vecx_{v}} \bigg)
\end{align}
Notice that $U_t$ is the same as the UCB estimate $\pucb$ defined in \cref{algo:ucb}.
Moreover, we define the ``good event" $\cF$ as
\begin{align}
\cF = \left \{
|x_{\tar}^T (\hat{\vecw}_{\so, t-1} - \vecw_{\so}^*)| \leq c \sqrt{x_{\tar}^T \Sigma_{\so, t-1}^{-1} x_{\tar}}, \, \forall \so, \tar  \in \cV, \, \forall t \leq T
\right \},
\end{align}
and the ``bad event" $\bar{\cF}$ as the complement of $\cF$. 

\subsection{Regret Decomposition}
Recall that the realized scaled regret at time $t$ is $R_t^{\rho \alpha}=\spr(\Sopt)-\frac{1}{\rho \alpha}\spr(\cS_t)$, thus we have
\begin{align}
R_t^{\rho \alpha}= & \, \spr(\Sopt)-\frac{1}{\rho \alpha}\spr(\cS_t) 
\stackrel{(a)}{=} \, \frac{1}{\rho} f(\tilde{\cS}, p^*) - \frac{1}{\rho \alpha}\spr(\cS_t) 
\stackrel{(b)}{\leq} \, \frac{1}{\rho} f(\tilde{\cS}, p^*) - \frac{1}{\rho \alpha} f(\cS_t, p^*),
\end{align}
where equality (a) follows from the definition of $\rho$ (i.e. $\rho$ is defined as $\rho=f(\tilde{\cS}, p^*)/\spr(\Sopt)$), and
inequality (b) follows from $f(\cS_t, p^*) \leq \spr(\cS_t)$ (see \cref{thm:mpr_appr}).
Thus, we have
\begin{align}
R^{\rho \alpha} (T) = & \, \E \left[\sum_{t=1}^T R_t^{\rho \alpha} \right] \nonumber \\
 \leq & \,  \frac{1}{\rho} \E \left \{ \sum_{t=1}^T \left[ f(\tilde{\cS}, p^*) -  f(\cS_t, p^*)/\alpha \right] \right \} \nonumber \\
 = & \,  \frac{P(\cF)}{\rho} \E \left \{ \sum_{t=1}^T \left[ f(\tilde{\cS}, p^*) -  f(\cS_t, p^*)/\alpha \right] \middle | \cF \right \} +
 \frac{P(\bar{\cF})}{\rho} \E \left \{ \sum_{t=1}^T \left[ f(\tilde{\cS}, p^*) -  f(\cS_t, p^*)/\alpha \right] \middle | \bar{\cF} \right \} \nonumber \\
 \leq & \, \frac{1}{\rho} \E \left \{ \sum_{t=1}^T \left[ f(\tilde{\cS}, p^*) -  f(\cS_t, p^*) /\alpha \right] \middle | \cF \right \} +
 \frac{P(\bar{\cF})}{\rho} nT,
\end{align}
where the last inequality follows from the naive bounds $P(\cF) \leq 1$ and $f(\tilde{\cS}, p^*) -  f(\cS_t, p^*)/ \alpha \leq n$.
Notice that under ``good" event $\cF$, we have
\begin{equation}
\label{eqn:append:LU}
L_t(u,v) \leq p^*_{uv} = x_v^T \theta^*_u \leq U_t(u,v) 
\end{equation}
for all node pair $(u,v)$ and for all time $t \leq T$. 
Thus, we have
$f(\cS, L_t) \leq f(\cS, p^*) \leq f(\cS, U_t)$ for all $\cS$ and $t \leq T$ under event $\cF$.
So under event $\cF$, we have
\[
f(\cS_t, L_t) \stackrel{(a)}{\leq} f(\cS_t, p^*) \stackrel{(b)}{\leq} f(\tilde{\cS},  p^*) \stackrel{(c)}{\leq} f(\tilde{\cS}, U_t) \leq 
\max_{\cS \in \cC} f(\cS, U_t) \stackrel{(d)}{\leq} \frac{1}{\alpha} f(\cS_t, U_t)
\]
for all $t \leq T$,
where inequalities (a) and (c) follow from \eqref{eqn:append:LU}, inequality (b) follows from
$\tilde{\cS} \in \argmax_{\cS \in \cC} f(\cS, p^*)$,
and inequality (d) follows from 
the fact that $\oracle$ is an $\alpha$-approximation algorithm.
Specifically, the fact that $\oracle$ is an $\alpha$-approximation algorithm implies that
$f(\cS_t, U_t) \geq \alpha \max_{\cS \in \cC} f(\cS, U_t)$.

Consequently, under event $\cF$, we have
\begin{align}
f(\tilde{\cS}, p^*) -  \frac{1}{\alpha}f(\cS_t, p^*) \leq & \, \frac{1}{\alpha} f(\cS_t, U_t) -  \frac{1}{\alpha} f(\cS_t, L_t) \nonumber \\
=& \, \frac{1}{\alpha} \sum_{v \in \cV} \left[\max_{u \in \cS_t} U_t(u,v) -\max_{u \in \cS_t} L_t(u,v)  \right] \nonumber \\
\leq & \, \frac{1}{\alpha} \sum_{v \in \cV} \sum_{u \in \cS_t} \left[U_t(u,v) - L_t(u,v)  \right] \nonumber \\
\leq & \, \sum_{v \in \cV} \sum_{u \in \cS_t} \frac{2c}{\alpha} \sqrt{x_{\tar}^T \Sigma_{\so, t-1}^{-1} x_{\tar}}.
\end{align}
So we have
\begin{align}
R^{\rho \alpha} (T) \leq
\frac{2c}{\rho \alpha} \E \left \{\sum_{t=1}^T \sum_{u \in \cS_t}  \sum_{v \in \cV}  \sqrt{x_{\tar}^T \Sigma_{\so, t-1}^{-1} x_{\tar}} \middle | \cF \right \} +
 \frac{P(\bar{\cF})}{\rho} nT.
\end{align}
In the remainder of this section, we will provide a worst-case bound on
$\sum_{t=1}^T \sum_{u \in \cS_t}  \sum_{v \in \cV}  \sqrt{x_{\tar}^T \Sigma_{\so, t-1}^{-1} x_{\tar}}$ (see \cref{sec:worst-case})
and a bound on the probability of ``bad event" $P(\bar{\cF})$ (see \cref{sec:bound_on_P}).

\subsection{Worst-Case Bound on $\sum_{t=1}^T \sum_{u \in \cS_t}  \sum_{v \in \cV}  \sqrt{x_{\tar}^T \Sigma_{\so, t-1}^{-1} x_{\tar}}$}
\label{sec:worst-case}
Notice that 
\[
\sum_{t=1}^T \sum_{u \in \cS_t}  \sum_{v \in \cV}  \sqrt{x_{\tar}^T \Sigma_{\so, t-1}^{-1} x_{\tar}}
=
\sum_{u \in \cV} \sum_{t=1}^T  \mathbf{1}\left[ u \in \cS_t \right] \sum_{v \in \cV}  \sqrt{x_{\tar}^T \Sigma_{\so, t-1}^{-1} x_{\tar}}
\]
For each $u \in \cV$, we define
$K_u = \sum_{t=1}^T  \mathbf{1}\left[ u \in \cS_t \right]$ 
as the number of times at which $u$ is chosen as a source node, then we have the following
lemma:
\begin{lemma}
For all $u \in \cV$, we have
\[ \sum_{t=1}^T  \mathbf{1}\left[ u \in \cS_t \right] \sum_{v \in \cV}  \sqrt{x_{\tar}^T \Sigma_{\so, t-1}^{-1} x_{\tar}} \leq 
 \sqrt{n K_u} \sqrt{\frac{dn \log \left ( 1 + \frac{n K_u}{d \lambda \sigma^2} \right )}{ \lambda \log \left(1+\frac{1}{\lambda \sigma^2} \right)}}.
\]
Moreover, when $X=I$, we have
\[
\sum_{t=1}^T  \mathbf{1}\left[ u \in \cS_t \right] \sum_{v \in \cV}  \sqrt{x_{\tar}^T \Sigma_{\so, t-1}^{-1} x_{\tar}} \leq 
\sqrt{n K_u} \sqrt{\frac{n \log \left ( 1 + \frac{K_u}{ \lambda \sigma^2} \right )}{ \lambda \log \left(1+\frac{1}{\lambda \sigma^2} \right)}}.
\]
\end{lemma}
\begin{proof}
To simplify the exposition, we use $\Sigma_t$ to denote $\Sigma_{\so, t}$, and define
$z_{t,v} =  \sqrt{x_{\tar}^T \Sigma_{\so, t-1}^{-1} x_{\tar}}$ for all $t \leq T$ and all $v \in \cV$.
Recall that
\[
\Sigma_t = \Sigma_{t-1} +  \frac{\mathbf{1}\left[ u \in \cS_t \right]}{\sigma^2} X X^T
=  \Sigma_{t-1} +  \frac{\mathbf{1}\left[ u \in \cS_t \right]}{\sigma^2} \sum_{v \in \cV} x_{\tar} x_{\tar}^T.
\]
Note that if $ u \notin \cS_t$, $\Sigma_t = \Sigma_{t-1}$.
If $u \in \cS_t$, then for any $\tar \in \cV$, we have
\begin{align}
\det \left[ \Sigma_t \right] \geq & \, \det \left[\Sigma_{t-1} + \frac{1}{\sigma^2}   x_{\tar} x_{\tar}^T \right ] \nonumber \\
= & \, \det \left[ \Sigma_{t-1}^{\frac{1}{2}} \left( I + \frac{1}{\sigma^2}   \Sigma_{t-1}^{-\frac{1}{2}} x_{\tar} x_{\tar}^T \Sigma_{t-1}^{-\frac{1}{2}} \right)\Sigma_{t-1}^{\frac{1}{2}}\right] \nonumber \\
= & \, \det \left[\Sigma_{t-1} \right]
\det \left[  I + \frac{1}{\sigma^2}   \Sigma_{t-1}^{-\frac{1}{2}} x_{\tar} x_{\tar}^T \Sigma_{t-1}^{-\frac{1}{2}} \right] \nonumber \\
= & \, \det \left[\Sigma_{t-1} \right] \left( 1+ \frac{1}{\sigma^2}  x_{\tar}^T \Sigma_{t-1}^{-1} x_{\tar} \right) =
\det \left[\Sigma_{t-1} \right] \left( 1+ \frac{z_{t-1, \tar}^2}{\sigma^2}  \right). \nonumber
\end{align}
Hence, we have
\begin{equation}
\det \left[ \Sigma_t \right]^n \geq \det \left[\Sigma_{t-1} \right]^n \prod_{\tar \in \cV} \left( 1+ \frac{z_{t-1, \tar}^2}{\sigma^2}  \right).
\end{equation}
Note that the above inequality holds for any $X$. However, if $X=I$, then all $\Sigma_t$'s are diagonal and we have
\begin{equation}
\label{eqn:append:worst:tabular}
\det \left[ \Sigma_t \right] = \det \left[\Sigma_{t-1} \right] \prod_{\tar \in \cV} \left( 1+ \frac{z_{t-1, \tar}^2}{\sigma^2}  \right).
\end{equation}
As we will show later, this leads to a tighter regret bound in the tabular ($X=I$) case.

Let's continue our analysis for general $X$. The above results imply that
\begin{equation}
n \log \left( \det \left[ \Sigma_t \right ]\right) \geq n \log \left( \det \left[ \Sigma_{t-1} \right ]\right) + \mathbf{1} \left( u \in \cS_t \right) \sum_{\tar \in \cV} \log \left( 1+ \frac{z_{t-1, \tar}^2}{\sigma^2}  \right) \nonumber
\end{equation}
and hence
\begin{align}
n \log \left( \det \left[ \Sigma_T \right ]\right) \geq & \, n \log \left( \det \left[ \Sigma_{0} \right ]\right) + \sum_{t=1}^T \mathbf{1} \left( u \in \cS_t \right) \sum_{\tar \in \cV} \log \left( 1+ \frac{z_{t-1, \tar}^2}{\sigma^2}  \right) \nonumber \\
= & \, nd \log(\lambda) + \sum_{t=1}^T \mathbf{1} \left( u \in \cS_t \right) \sum_{\tar \in \cV} \log \left( 1+ \frac{z_{t-1, \tar}^2}{\sigma^2}  \right).
\end{align}
On the other hand, we have that
\begin{align}
\Tr \left[ \Sigma_T \right] =& \, \Tr \left[ \Sigma_0 + \sum_{t=1}^T \frac{\mathbf{1}\left[ u \in \cS_t \right]}{\sigma^2} \sum_{\tar \in \cV} x_{\tar} x_{\tar}^T \right]
\nonumber \\
=& \, \Tr \left[ \Sigma_0 \right] + \sum_{t=1}^T \frac{\mathbf{1}\left[ u \in \cS_t \right]}{\sigma^2} \sum_{\tar \in \cV} \Tr \left[ x_{\tar} x_{\tar}^T \right]
\nonumber \\
=& \, \lambda d +  \sum_{t=1}^T \frac{\mathbf{1}\left[ u \in \cS_t \right]}{\sigma^2} \sum_{\tar \in \cV} \|  x_{\tar} \|^2 \leq
\lambda d +\frac{n K_u}{\sigma^2},
\end{align}
where the last inequality follows from the assumption that $\|x_{\tar} \| \leq 1$ and the definition of $K_u$.
From the trace-determinant inequality, we have $\frac{1}{d} \Tr \left[ \Sigma_T \right] \geq \det \left[ \Sigma_T \right]^{\frac{1}{d}}$. Thus,
we have
\[
dn \log \left ( \lambda + \frac{n K_u}{d \sigma^2} \right ) \geq
dn \log \left ( \frac{1}{d} \Tr \left[ \Sigma_T \right]  \right ) \geq n \log \left( \det \left[ \Sigma_T \right] \right)
\geq  dn \log(\lambda) + \sum_{t=1}^T \mathbf{1} \left( u \in \cS_t \right) \sum_{\tar \in \cV} \log \left( 1+ \frac{z_{t-1, \tar}^2}{\sigma^2}  \right).
\]
That is
\[
\sum_{t=1}^T \mathbf{1} \left( u \in \cS_t \right) \sum_{\tar \in \cV} \log \left( 1+ \frac{z_{t-1, \tar}^2}{\sigma^2}  \right) \leq 
dn \log \left ( 1 + \frac{n K_u}{d \lambda \sigma^2} \right ) 
\]
Notice that $z^2_{t-1, \tar} = x_{\tar}^T \Sigma_{t-1}^{-1} x_{\tar} \leq x_{\tar}^T \Sigma_{0}^{-1} x_{\tar} = \frac{\|x_{\tar} \|^2}{\lambda} \leq 
\frac{1}{\lambda}$. Moreover, for all $y \in [0, 1/\lambda]$, we have
$\log \left(1+\frac{y}{\sigma^2} \right) \geq \lambda \log \left(1+\frac{1}{\lambda \sigma^2} \right) y$ based on the concavity of 
$\log (\cdot)$. Thus, we have
\[
 \lambda \log \left(1+\frac{1}{\lambda \sigma^2} \right) \sum_{t=1}^T \mathbf{1} \left( u \in \cS_t \right) \sum_{\tar \in \cV} z_{t-1, \tar}^2 \leq 
 dn \log \left ( 1 + \frac{n K_u}{d \lambda \sigma^2} \right ) .
\]
Finally, from Cauchy-Schwarz inequality, we have that
\[
\sum_{t=1}^T \mathbf{1} \left( u \in \cS_t \right) \sum_{\tar \in \cV} z_{t-1, \tar}
\leq \sqrt{n K_u} \sqrt{\sum_{t=1}^T \mathbf{1} \left( u \in \cS_t \right) \sum_{\tar \in \cV} z_{t-1, \tar}^2}.
\] 
Combining the above results, we have
\begin{equation}
\sum_{t=1}^T \mathbf{1} \left( u \in \cS_t \right) \sum_{\tar \in \cV} z_{t-1, \tar}
\leq \sqrt{n K_u} \sqrt{\frac{dn \log \left ( 1 + \frac{n K_u}{d \lambda \sigma^2} \right )}{ \lambda \log \left(1+\frac{1}{\lambda \sigma^2} \right)}}.
\end{equation}
This concludes the proof for general $X$. Based on \eqref{eqn:append:worst:tabular}, the analysis for the tabular ($X=I$) case
is similar, and we omit the detailed analysis. In the tabular case, we have
\begin{equation}
\sum_{t=1}^T \mathbf{1} \left( u \in \cS_t \right) \sum_{\tar \in \cV} z_{t-1, \tar}
\leq \sqrt{n K_u} \sqrt{\frac{n \log \left ( 1 + \frac{K_u}{ \lambda \sigma^2} \right )}{ \lambda \log \left(1+\frac{1}{\lambda \sigma^2} \right)}}.
\end{equation}
\end{proof}
We now develop a worst-case bound. Notice that for general $X$, we have
\begin{align}
\sum_{u \in \cV} \sum_{t=1}^T  \mathbf{1}\left[ u \in \cS_t \right] \sum_{v \in \cV}  \sqrt{x_{\tar}^T \Sigma_{\so, t-1}^{-1} x_{\tar}} 
\leq & \,
\sum_{u \in \cV} \sqrt{n K_u} \sqrt{\frac{dn \log \left ( 1 + \frac{n K_u}{d \lambda \sigma^2} \right )}{ \lambda \log \left(1+\frac{1}{\lambda \sigma^2} \right)}}
\nonumber \\
\stackrel{(a)}{\leq} & \, 
n \sqrt{\frac{d \log \left ( 1 + \frac{n T}{d \lambda \sigma^2} \right )}{ \lambda \log \left(1+\frac{1}{\lambda \sigma^2} \right)}} \sum_{u \in \cV} \sqrt{ K_u} 
\nonumber \\
\stackrel{(b)}{\leq} & \, n \sqrt{\frac{d \log \left ( 1 + \frac{n T}{d \lambda \sigma^2} \right )}{ \lambda \log \left(1+\frac{1}{\lambda \sigma^2} \right)}} \sqrt{n} \sqrt{\sum_{u \in \cV} K_u} \nonumber \\
\stackrel{(c)}{=} & n^{\frac{3}{2}} \sqrt{\frac{d K T \log \left ( 1 + \frac{n T}{d \lambda \sigma^2} \right )}{ \lambda \log \left(1+\frac{1}{\lambda \sigma^2} \right)}},
\end{align}
where inequality (a) follows from the naive bound $K_u \leq T$, inequality (b) follows from Cauchy-Schwarz inequality, and equality (c) follows from
$\sum_{\so \in \cV} K_u = KT$. Similarly, for the special case with $X=I$, we have
\begin{align}
\sum_{u \in \cV} \sum_{t=1}^T  \mathbf{1}\left[ u \in \cS_t \right] \sum_{v \in \cV}  \sqrt{x_{\tar}^T \Sigma_{\so, t-1}^{-1} x_{\tar}} 
\leq & \,
\sum_{u \in \cV} \sqrt{n K_u} \sqrt{\frac{n \log \left ( 1 + \frac{K_u}{ \lambda \sigma^2} \right )}{ \lambda \log \left(1+\frac{1}{\lambda \sigma^2} \right)}}
\leq  \, 
n^{\frac{3}{2}} \sqrt{ \frac{KT \log \left ( 1 + \frac{T}{ \lambda \sigma^2} \right )}{ \lambda \log \left(1+\frac{1}{\lambda \sigma^2} \right)}} .
\end{align}
This concludes the derivation of a worst-case bound.

\subsection{Bound on $P \left( \bar{\cF}\right)$}
\label{sec:bound_on_P}
We now derive a bound on $P \left( \bar{\cF}\right)$ based on the ``Self-Normalized Bound for Matrix-Valued Martingales" developed in~\cref{thm:self-normalization} (see \cref{thm:self-normalization}). Before proceeding, we define
$\cF_\so$ for all $\so \in \cV$ as
\begin{align}
\cF_{\so} = \left \{
|x_{\tar}^T (\hat{\vecw}_{\so, t-1} - \vecw_{\so}^*)| \leq c \sqrt{x_{\tar}^T \Sigma_{\so, t-1}^{-1} x_{\tar}}, \, \forall  \tar  \in \cV, \, \forall t \leq T
\right \},
\end{align}
and the  $\bar{\cF}_\so$ as the complement of $\cF_\so$. Note that by definition, $\bar{\cF}=\bigcup_{\so \in \cV} \bar{\cF}_\so$.
Hence, we first develop a bound on $P \left( \bar{\cF}_\so \right)$, then we develop a bound on
$P \left( \bar{\cF} \right)$ based on union bound.
\begin{lemma}
For all $\so \in \cV$, all $ \sigma, \lambda>0$, all $ \delta \in (0,1)$, and all
\[
c \geq  \frac{1}{\sigma} \sqrt{dn \log \left( 1+ \frac{nT}{\sigma^2 \lambda d} \right)+ 2 \log \left(\frac{1}{\delta} \right)}  + \sqrt{\lambda}  \|   \vecw^*_{\so} \|_2
\]
we have $P \left( \bar{\cF}_\so \right) \leq \delta$.
\end{lemma}
\begin{proof}
To simplify the expositions, we omit the subscript $u$ in this proof. For instance, we use $\mathbf{\theta}^*$, $\Sigma_t$, $\mathbf{y}_t$ and
$\mathbf{b}_t$ to respectively denote $\mathbf{\theta}^*_{\so}$, $\Sigma_{\so, t}$, $\mathbf{y}_{\so, t}$ and
$\mathbf{b}_{\so, t}$. We also use $\cH_t$ to denote the ``history" by the end of time $t$, and hence
$ \left \{ \cH_t \right \}_{t=0}^{\infty}$ is a filtration. Notice that $U_t$ is $\cH_{t-1}$-adaptive, and hence
$\cS_t$ and $\mathbf{1}\left[ u \in \cS_t \right]$ are also $\cH_{t-1}$-adaptive. We define
\begin{equation}
\eta_t = \left \{
\begin{array}{ll}
\mathbf{y}_t - X^T \mathbf{\theta}^* & \textrm{if $u \in \cS_t$} \\
0 & \textrm{otherwise}
\end{array}
\right. \in \Re^n
\quad
\textrm{ and }
\quad 
X_t = \left \{
\begin{array}{ll}
X & \textrm{if $u \in \cS_t$} \\
0 & \textrm{otherwise}
\end{array}
\right. \in \Re^{d \times n}
\end{equation}
Note that $X_t$ is $\cH_{t-1}$-adaptive, and $\eta_t$ is $\cH_t$-adaptive.
Moreover, $\| \eta_t\|_{\infty} \leq 1$ always holds, and 
$ \E \left[ \eta_t \middle | \cH_{t-1} \right]=0$. To simplify the expositions, we further define $\mathbf{y}_t=0$ for all $t$
s.t. $u \notin \cS_t$. Note that with this definition, we have $\eta_t = \mathbf{y}_t - X_t^T \mathbf{\theta}^* $ for all $t$.
We further define
\begin{align}
\bar{V}_t =& \, n \sigma^2 \Sigma_t = n \sigma^2 \lambda I + n \sum_{s=1}^t X_s X_s^T \nonumber \\
\bar{S}_t = & \, \sum_{s=1}^t X_s \eta_s = \sum_{s=1}^t X_s \left[\mathbf{y}_s - X_s^T \mathbf{\theta}^* \right]
=\mathbf{b}_t - \sigma^2 \left[ \Sigma_t - \lambda I \right] \mathbf{\theta}^*
\end{align}
Thus, we have $\Sigma_t \hat{\vecw}_t = \sigma^{-2} \mathbf{b}_t = \sigma^{-2} \bar{S}_t + \left[ \Sigma_t - \lambda I \right] \vecw^*$,
which implies 
\begin{align}
\hat{\vecw}_t  - \vecw^* = \Sigma_t^{-1} \left[\sigma^{-2} \bar{S}_t - \lambda   \vecw^* \right].
\end{align}
Consequently, for any $\tar \in \cV$, we have
\begin{align}
\left | x_\tar^T \left( \hat{\vecw}_t  - \vecw^* \right) \right | =& \,
\left | x_\tar^T  \Sigma_t^{-1} \left[\sigma^{-2} \bar{S}_t - \lambda   \vecw^* \right] \right | \leq 
\sqrt{x_v^T \Sigma_t^{-1} x_v} \|\sigma^{-2} \bar{S}_t - \lambda   \vecw^* \|_{\Sigma_t^{-1}} \nonumber \\
\leq & \, 
\sqrt{x_v^T \Sigma_t^{-1} x_v} \left[ \|\sigma^{-2} \bar{S}_t  \|_{\Sigma_t^{-1}} + \| \lambda   \vecw^* \|_{\Sigma_t^{-1}}  \right],
\end{align}
where the first inequality follows from Cauchy-Schwarz inequality and the second inequality follows from triangular inequality.
Note that $ \| \lambda   \vecw^* \|_{\Sigma_t^{-1}} =  \lambda  \|   \vecw^* \|_{\Sigma_t^{-1}} \leq \lambda  \|   \vecw^* \|_{\Sigma_0^{-1}} = \sqrt{\lambda}  \|   \vecw^* \|_2$. On the other hand, since $\Sigma_t^{-1} = n \sigma^2 \bar{V}_t^{-1}$, we have
$\|\sigma^{-2} \bar{S}_t  \|_{\Sigma_t^{-1}}  = \frac{\sqrt{n}}{\sigma} \| \bar{S}_t  \|_{\bar{V}_t^{-1}}$.
Thus, we have
\begin{equation}
\left | x_\tar^T \left( \hat{\vecw}_t  - \vecw^* \right) \right |  \leq 
\sqrt{x_v^T \Sigma_t^{-1} x_v} \left[  \frac{\sqrt{n}}{\sigma} \| \bar{S}_t  \|_{\bar{V}_t^{-1}} + \sqrt{\lambda}  \|   \vecw^* \|_2 \right].
\end{equation}
From \cref{thm:self-normalization}, we know with probability at least $1-\delta$, for all $t \leq T$, we have
\[
\| S_t\|_{\bar{V}_t^{-1}}^2 \leq 2 \log \left( \frac{\det \left( \bar{V}_{t}\right)^{1/2} \det \left( V\right)^{-1/2} }{\delta}\right)
 \leq 2 \log \left( \frac{\det \left( \bar{V}_{T}\right)^{1/2} \det \left( V\right)^{-1/2} }{\delta}\right),
\]
where $V= n \sigma^2 \lambda I$. Note that from the trace-determinant inequality, we have
\[
\det \left[ \bar{V}_T \right]^{\frac{1}{d}} \leq \frac{\Tr \left[  \bar{V}_T\right]}{d} \leq \frac{n \sigma^2 \lambda d +n^2 T}{d},
\]
where the last inequality follows from $\Tr \left[ X_t X_t^T \right] \leq n$ for all $t$. Note that $\det \left[ V\right] = \left[ n \sigma^2 \lambda \right]^d$, with a little bit algebra, we have
\[
\| S_t\|_{\bar{V}_t^{-1}} \leq \sqrt{d \log \left( 1+ \frac{nT}{\sigma^2 \lambda d} \right)+ 2 \log \left(\frac{1}{\delta} \right)} \quad \forall t \leq T
\]
with probability at least $1-\delta$. Thus, if
\[
c \geq  \frac{1}{\sigma} \sqrt{dn \log \left( 1+ \frac{nT}{\sigma^2 \lambda d} \right)+ 2 \log \left(\frac{1}{\delta} \right)}  + \sqrt{\lambda}  \|   \vecw^* \|_2,
\]
then $\cF_u$ holds with probability at least $1-\delta$. This concludes the proof of this lemma.
\end{proof}
Hence, from the union bound, we have the following lemma:
\begin{lemma}
For all $ \sigma, \lambda>0$, all $ \delta \in (0,1)$, and all
\begin{equation}
\label{append:c_lower}
c \geq  \frac{1}{\sigma} \sqrt{dn \log \left( 1+ \frac{nT}{\sigma^2 \lambda d} \right)+ 2 \log \left(\frac{n}{\delta} \right)}  + \sqrt{\lambda}  \max_{u \in \cV} \|   \vecw^*_\so \|_2
\end{equation}
we have $P \left( \bar{\cF} \right) \leq \delta$.
\end{lemma}
\begin{proof}
This lemma follows directly from the union bound. Note that for all $c$ satisfying Equation~\ref{append:c_lower}, we have
$P \left ( \bar{\cF}_\so \right) \leq \frac{\delta}{n}$ for all $\so \in \cV$, which implies
$P \left ( \bar{\cF} \right) = P \left ( \bigcup_{\so \in \cV}\bar{\cF}_\so \right) \leq \sum_{\so \in \cV} P \left ( \bar{\cF}_\so \right)  \leq \delta$.
\end{proof}

\subsection{Conclude the Proof}
Note that if we choose
\begin{equation}
c \geq  \frac{1}{\sigma} \sqrt{dn \log \left( 1+ \frac{nT}{\sigma^2 \lambda d} \right)+ 2 \log \left(n^2 T \right)}  + \sqrt{\lambda}  \max_{u \in \cV} \|   \vecw^*_\so \|_2, 
\end{equation}
we have $P \left( \bar{\cF} \right) \leq \frac{1}{nT}$. Hence for general $X$, we have
\begin{align}
R^{\rho \alpha} (T) \leq & \,
\frac{2c}{\rho \alpha} \E \left \{\sum_{t=1}^T \sum_{u \in \cS_t}  \sum_{v \in \cV}  \sqrt{x_{\tar}^T \Sigma_{\so, t-1}^{-1} x_{\tar}} \middle | \cF \right \} +
 \frac{1}{\rho} \nonumber \\
 \leq & \, \frac{2c}{\rho \alpha} n^{\frac{3}{2}} \sqrt{\frac{d K T \log \left ( 1 + \frac{n T}{d \lambda \sigma^2} \right )}{ \lambda \log \left(1+\frac{1}{\lambda \sigma^2} \right)}}+ \frac{1}{\rho} .
\end{align}
Note that with $c= \frac{1}{\sigma} \sqrt{dn \log \left( 1+ \frac{nT}{\sigma^2 \lambda d} \right)+ 2 \log \left(n^2 T \right)}  + \sqrt{\lambda}  \max_{u \in \cV} \|   \vecw^*_\so \|_2$, this regret bound is
$\tilde{O} \left(
\frac{n^2 d \sqrt{KT}}{\rho \alpha}
\right)$. 
Similarly, for the special case $X=I$, we have
\begin{align}
R^{\rho \alpha} (T) 
 \leq & \, \frac{2c}{\rho \alpha} n^{\frac{3}{2}} \sqrt{ \frac{KT \log \left ( 1 + \frac{T}{ \lambda \sigma^2} \right )}{ \lambda \log \left(1+\frac{1}{\lambda \sigma^2} \right)}} + \frac{1}{\rho} .
\end{align}
Note that with $c= \frac{n}{\sigma} \sqrt{ \log \left( 1+ \frac{T}{\sigma^2 \lambda } \right)+ 2 \log \left(n^2 T \right)}  + \sqrt{\lambda}  \max_{u \in \cV} \|   \vecw^*_\so \|_2 \leq \frac{n}{\sigma} \sqrt{ \log \left( 1+ \frac{T}{\sigma^2 \lambda } \right)+ 2 \log \left(n^2 T \right)}  + \sqrt{\lambda n}  $, this regret bound is
$\tilde{O} \left(
\frac{n^{\frac{5}{2}}  \sqrt{KT}}{\rho \alpha}
\right)$. 

\section{Self-Normalized Bound for Matrix-Valued Martingales}
\label{sec:matrix_conc}
In this section, we derive a ``self-normalized bound" for matrix-valued Martingales. 
This result is a natural generalization of Theorem 1 in \citet{abbasi2011improved}.
\begin{theorem}
\label{thm:self-normalization}
\textrm{(Self-Normalized Bound for Matrix-Valued Martingales)} Let $ \left \{ \cH_t \right \}_{t=0}^{\infty}$ be a filtration, and
$ \left \{ \eta_t \right \}_{t=1}^{\infty} $ be a $\Re^K$-valued Martingale difference sequence with respect to $ \left \{ \cH_t \right \}_{t=0}^{\infty}$.
Specifically, for all $t$, $\eta_t$ is $\cH_t$-measurable and satisfies (1) $\E \left[ \eta_t \middle | \cH_{t-1} \right]=0$
and (2) $\| \eta_t \|_{\infty} \leq 1$ with probability 1 conditioning on $\cH_{t-1}$.
Let $\left \{ X_t\right \}_{t=1}^{\infty}$ be a $\Re^{d \times K}$-valued stochastic process such that
$X_t$ is $\cH_{t-1}$ measurable. Assume that $V \in \Re^{d \times d}$ is a positive-definite matrix. For any $t \geq 0$, define
\begin{equation}
\bar{V}_t = V + K \sum_{s=1}^t X_s X_s^T \quad \quad S_t = \sum_{s=1}^t X_s \eta_s.
\end{equation}
Then, for any $\delta>0$, with probability at least $1-\delta$,  we have
\begin{equation}
\label{eqn:self-normalization}
\| S_t\|_{\bar{V}_t^{-1}}^2 \leq 2 \log \left( \frac{\det \left( \bar{V}_{t}\right)^{1/2} \det \left( V\right)^{-1/2} }{\delta}\right)  \quad \forall t \geq 0.
\end{equation}
\end{theorem}
We first define some useful notations. Similarly as \citet{abbasi2011improved}, for any $\lambda \in \Re^d$ and any $t$, we define $D_t^{\lambda}$
as
\begin{equation}
D_t^{\lambda}=\exp \left( \lambda^T X_t \eta_t -\frac{K}{2} \| X_t^T \lambda \|^2_2 \right),
\end{equation}
and $M_t^{\lambda} = \prod_{s=1}^t D_s^{\lambda}$ with convention $M_0^{\lambda} = 1$. Note that both $D_t^{\lambda}$ and $M_t^{\lambda}$ are 
$\cH_t$-measurable, and $ \left \{ M_t^{\lambda} \right \}_{t=0}^{\infty}$ is a supermartingale with respect to the filtration
$\left \{ \cH_t \right \}_{t=0}^{\infty}$. To see it, notice that conditioning on $\cH_{t-1}$, we have
\[
\lambda^T X_t \eta_t = (X_t^T \lambda)^T \eta_t \leq \| X_t^T \lambda \|_1 \| \eta_t \|_{\infty} \leq \| X_t^T \lambda \|_1 \leq 
\sqrt{K} \| X_t^T \lambda \|_2
\]
with probability $1$. This implies that $\lambda^T X_t \eta_t$ is conditionally $\sqrt{K} \| X_t^T \lambda \|_2$-subGaussian.
Thus, we have
\[
\E \left[ D_t^{\lambda} \middle | \cH_{t-1} \right] = \E \left[ \exp \left( \lambda^T X_t \eta_t  \right ) \middle | \cH_{t-1} \right] \exp \left( -\frac{K}{2} \| X_t^T \lambda \|^2_2 \right ) \leq \exp \left( \frac{K}{2} \| X_t^T \lambda \|^2_2  -\frac{K}{2} \| X_t^T \lambda \|^2_2 \right ) =1.
\]
Thus,
\[
\E \left[M_t^{\lambda} \middle | \cH_{t-1} \right] = M^{\lambda}_{t-1} \E \left[D_t^{\lambda} \middle | \cH_{t-1} \right]  \leq M^{\lambda}_{t-1} .
\]
So $ \left \{ M_t^{\lambda} \right \}_{t=0}^{\infty}$ is a supermartingale with respect to the filtration
$\left \{ \cH_t \right \}_{t=0}^{\infty}$. Then, following Lemma 8 of \citet{abbasi2011improved}, we have the following lemma:
\begin{lemma}
\label{lemma:conc_1}
Let $\tau$ be a stopping time with respect to the filtration $\left \{  \cH_t \right \}_{t=0}^{\infty}$. Then for any $\lambda \in \Re^d$,
$M_{\tau}^{\lambda}$ is almost surely well-defined and $\E \left[ M_\tau^{\lambda}\right] \leq 1$.
\end{lemma}
\begin{proof}
First, we argue that $M_{\tau}^{\lambda}$ is almost surely well-defined.
By Doob's convergence theorem for nonnegative supermartingales, 
$M_{\infty}^{\lambda}=\lim_{t \rightarrow \infty} M_t^{\lambda}$ is almost surely well-defined.
Hence $M_{\tau}^{\lambda}$ is indeed well-defined independent of $\tau<\infty$ or not.
Next, we show that $\E \left[ M_\tau^{\lambda}\right] \leq 1$. Let $Q_t^{\lambda}=M_{\min \{ \tau, t\}}^{\lambda}$ be
a stopped version of $\left \{ M_t^{\lambda} \right \}_{t=1}^{\infty}$. By Fatou's Lemma, we have
$\E \left[ M_\tau^{\lambda}\right] = \E \left[ \liminf_{t \rightarrow \infty} Q_t^{\lambda} \right] \leq \liminf_{t \rightarrow \infty} 
\E \left[ Q_t^{\lambda} \right] \leq 1$.
\end{proof}

The following results follow from Lemma 9 of \citet{abbasi2011improved}, which uses the ``method of mixtures" technique.
Let $\Lambda$ be a Gaussian random vector in
$\Re^d$ with mean $0$ and covariance matrix $V^{-1}$, and independent of all the other random variables.
Let $\cH_{\infty}$ be the tail $\sigma$-algebra of the filtration,
i.e. the $\sigma$-algebra generated by the union of all events in the filtration.
We further define $M_t = \E \left [ M_t^{\Lambda} \middle | \cH_{\infty} \right]$ for all $t=0,1,\ldots$ and $t=\infty$.
Note that $M_{\infty}$ is almost surely well-defined since $M_{\infty}^{\lambda}$ is almost surely well-defined.

Let $\tau$ be a stopping time with respect to the filtration $\left \{ \cH_t \right \}_{t=0}^{\infty}$. Note that $M_{\tau}$ is almost surely well-defined since
$M_{\infty}$ is almost surely well-defined. Since $\E \left[ M_{\tau}^{\lambda} \right] \leq 1$ from  Lemma~\ref{lemma:conc_1}, we have
\[
\E \left[ M_{\tau} \right] = \E \left [ M_{\tau}^{\Lambda} \right] = \E \left [ \E \left[ M_{\tau}^{\Lambda} \middle | \Lambda \right] \right] \leq 1.
\]
The following lemma follows directly from the proof for Lemma 9 of \citet{abbasi2011improved}, which can be derived by algebra.
The proof is omitted here.
\begin{lemma}
\label{lemma:conc_2}
For all finite $t=0,1,\ldots$, we have
\begin{equation}
M_t = \left( \frac{\det(V)}{\det (\bar{V}_t)}\right)^{1/2} \exp \left( \frac{1}{2} \| S_t \|_{\bar{V}_t^{-1}}\right).
\end{equation}
\end{lemma}
Note that Lemma~\ref{lemma:conc_2} implies that for finite $t$,
$\| S_t \|^2_{\bar{V}_{t}^{-1}} > 2 \log \left( \frac{\det \left( \bar{V}_{t}\right)^{1/2} \det \left( V\right)^{-1/2} }{\delta}\right) $
and $M_{t} > \frac{1}{\delta}$ are equivalent. Consequently, for any stopping time $\tau$, the event
\[
\left \{  \tau<\infty, \, \| S_\tau \|^2_{\bar{V}_{\tau}^{-1}} > 2 \log \left( \frac{\det \left( \bar{V}_{\tau}  \right)^{1/2} \det \left( V\right)^{-1/2} }{\delta}\right) 
\right \} 
\]
is equivalent to $\left \{  \tau<\infty, \, M_{\tau} > \frac{1}{\delta} \right \}$.
Finally, we prove Theorem~\ref{thm:self-normalization}:

\begin{proof}
We define the ``bad event" at time $t=0,1,\ldots$ as:
\[
B_t (\delta)=\left \{
\| S_t\|^2_{\bar{V}_{t}^{-1}} > 2 \log \left( \frac{\det \left( \bar{V}_{t}\right)^{1/2} \det \left( V\right)^{-1/2} }{\delta}\right) 
 \right \}.
\]
We are interested in bounding the probability of the ``bad event" $\bigcup_{t=1}^{\infty} B_t (\delta)$. Let $\Omega$ 
denote the sample space, for any outcome $\omega \in \Omega$, we define
$\tau (\omega) = \min \{ t \geq 0: \, \omega \in B_t (\delta) \}$, with the convention that $\min \emptyset = +\infty$.
Thus, $\tau$ is a stopping time. Notice that
$\bigcup_{t=1}^{\infty} B_t (\delta) = \{ \tau < \infty\}$. Moreover, if $\tau<\infty$, then by definition of $\tau$, we have
$\| S_\tau\|^2_{\bar{V}_{\tau}^{-1}} > 2 \log \left( \frac{\det \left( \bar{V}_{\tau}\right)^{1/2} \det \left( V\right)^{-1/2} }{\delta}\right) $,
which is equivalent to $M_{\tau} > \frac{1}{\delta}$ as discussed above.
Thus we have
\begin{align*}
P \left( \bigcup_{t=1}^{\infty} B_t (\delta) \right) \stackrel{(a)}{=}& \, P \left(  \tau < \infty \right) \nonumber \\
\stackrel{(b)}{=} & \, P \left (  \| S_\tau  \|^2_{\bar{V}_{\tau}^{-1}} > 2 \log \left( \frac{\det \left( \bar{V}_{\tau} \right)^{1/2} \det \left( V\right)^{-1/2} }{\delta}\right), \, \tau < \infty \right) \nonumber \\
 \stackrel{(c)}{=} & \, P \left(  M_{\tau}> 1/\delta, \,\tau < \infty \right) \nonumber \\
\leq & \, P \left(  M_{\tau}> 1/\delta \right) \nonumber \\
\stackrel{(d)}{\leq}  & \, \delta,
\end{align*}
where equalities (a) and (b) follow from the definition of $\tau$, equality (c) follows from Lemma~\ref{lemma:conc_2}, and inequality (d) follows from Markov's inequality.
This concludes the proof for Theorem~\ref{thm:self-normalization}.
\end{proof}

We conclude this section by briefly discussing a special case. If for any $t$, the elements of $\eta_t$ are statistically independent conditioning on $\cH_{t-1}$, then we can prove a variant of Theorem~\ref{thm:self-normalization}:
with $\bar{V}_t = V +  \sum_{s=1}^t X_s X_s^T$ and $S_t = \sum_{s=1}^t X_s \eta_s$,
Equation~\ref{eqn:self-normalization} holds with probability at least $1-\delta$. To see it, notice that in this case
 \begin{align}
 \E \left[ \exp \left( \lambda^T X_t \eta_t \right) \middle | \cH_{t-1} \right] =& \,  \E \left[ \prod_{k=1}^K \exp \left(  (X_t^T \lambda)(k) \eta_t(k) \right) \middle | \cH_{t-1} \right] \nonumber \\
 \stackrel{(a)}{=} & \,  \prod_{k=1}^K \E \left[ \exp \left(  (X_t^T \lambda)(k) \eta_t(k) \right) \middle | \cH_{t-1} \right] \nonumber \\
 \stackrel{(b)}{\leq} & \, \prod_{k=1}^K \exp \left( \frac{(X_t^T \lambda)(k)^2}{2} \right) = \exp \left( \frac{ \left \| X_t^T \lambda \right \|^2}{2} \right),
 \end{align}
 where $(k)$ denote the $k$-th element of the vector. Note that the equality (a) follows from the conditional independence of the elements
 in $\eta_t$, and inequality (b) follows from $|\eta_t(k) | \leq 1$ for all $t$ and $k$. Thus, if we redefine
$
D_t^{\lambda}=\exp \left( \lambda^T X_t \eta_t -\frac{1}{2} \| X_t^T \lambda \|^2_2 \right)$,
and $M_t^{\lambda} = \prod_{s=1}^t D_s^{\lambda}$, we can prove that $\{ M_t^{\lambda}\}_{t}$ is a supermartingale.
Consequently, using similar analysis techniques, we can prove the variant of Theorem~\ref{thm:self-normalization}
discussed in this paragraph.

\section{Laplacian Regularization}
\label{sec:lap-reg}
As explained in section~\ref{sec:implementation}, enforcing Laplacian regularization leads to the following optimization problem:
\begin{flalign*}
\hat{\vecw}_{t} = \argmin_{\vecw}[ \sum_{j = 1}^{t} \sum_{\so \in \cS_{t}}(y_{\so,j} - {\vecw}_{\so} X)^{2} + \lambda_{2} \sum_{(\so_{1},\so_{2}) \in \cE} \vert \vert {\vecw}_{\so_{1}} - {\vecw}_{\so_{2}} \vert \vert_{2}^{2} ]
\end{flalign*}
Here, the first term is the data fitting term, whereas the second term is the Laplacian regularization terms which enforces smoothness in the source node estimates. This can optimization problem can be re-written as follows:
\begin{align}
\hat{\vecw}_{t} = \argmin_{\vecw} \bigg[\sum_{j = 1}^{t} \sum_{\so \in \cS_{t}}(y_{\so,j} - {\vecw}_{\so} X)^{2} \nonumber + \lambda_{2} \vecw^{T} (L \otimes I_{d}) \vecw  \bigg]
\label{eq:lr-1}
\end{align}
Here, $\vecw \in \Re^{dn}$ is the concatenation of the $n$ $d$-dimensional $\vecw_{\so}$ vectors and $A \otimes B$ refers to the Kronecker product of matrices $A$ and $B$. Setting the gradient of equation~\ref{eq:lr-1} to zero results in solving the following linear system:
\begin{align}
[X X^{T} \otimes I_{n} + \lambda_{2} L \otimes I_{d}] \hat{\vecw}_{t}  = b_{t}
\end{align}
Here $b_{t}$ corresponds to the concatenation of the $n$ $d$-dimensional vectors $b_{u,t}$. This is the Sylvester equation and there exist sophisticated methods of solving it. For simplicity, we focus on the special case when the features are derived from the Laplacian eigenvectors (Section~\ref{sec:implementation}). 

Let $\beta_{t}$ be a diagonal matrix such that $\beta_{t}{\so,\so}$ refers to the number of times node $\so$ has been selected as the source. Since the Laplacian eigenvectors are orthogonal, when using Laplacian features, $X X^{T} \otimes I_{n} = \beta \otimes I_{d}$. We thus obtain the following system:
\begin{align}
[ \left( \beta + \lambda_{2} L \right) \otimes I_{d}] \hat{\vecw}_{t}  = b_{t} 
\end{align}
Note that the matrix $ \left( \beta + \lambda_{2} L \right)$ and thus $ \left( \beta + \lambda_{2} L \right) \otimes I_{d}$ is positive semi-definite and can be solved using conjugate gradient~\cite{hestenes1952methods}. 

For conjugate gradient, the most expensive operation is the matrix-vector multiplication $\left( \beta + \lambda_{2} L \right) \otimes I_{d}] \mathbf{v}$ for an arbitrary vector $\mathbf{v}$. Let $\vect$ be an operation that takes a $d \times n$ matrix and stacks it column-wise converting it into a $dn$-dimensional vector. Let $V$ refer to the $d \times n$ matrix obtained by partitioning the vector $\mathbf{v}$ into columns of $V$. Given this notation, we use the property that $(B^{T} \otimes A) \mathbf{v} = vec(AVB)$. This implies that the matrix-vector multiplication can then be rewritten as follows: 
\begin{align}
\left( \beta + \lambda_{2} L \right) \otimes I_{d}  \mathbf{v} = \vect(V \left( \beta + \lambda_{2} L^{T} \right))
\end{align}

Since $\beta$ is a diagonal matrix, $V \beta$ is an $O(dn)$ operation, whereas $V L^{T}$ is an $O(dm)$ operation since there are only $m$ non-zeros (corresponding to edges) in the Laplacian matrix. Hence the complexity of computing the mean $\hat{\vecw}_{t}$ is an order $O((d(m + n)) \kappa)$ where $\kappa$ is the number of conjugate gradient iterations. In our experiments, similar to~\cite{vaswani2017horde},  we warm-start with the solution at the previous round and find that $\kappa = 5$ is enough for convergence.

Unlike independent estimation where we update the UCB estimates for only the selected nodes, when using Laplacian regularization, the upper confidence values for each reachability probability need to be recomputed in each round. Once we have an estimate of $\vecw$, calculating the mean estimates for the reachabilities for all $u,v$ requires $O(d n^{2})$ computation. This is the most expensive step when using Laplacian regularization. 

We now describe how to compute the confidence intervals. For this, let $\vecD$ denote the diagonal of $(\beta + \lambda_{2} L)^{-1}$. The UCB value $\CB_{u,v,t}$ can then be computed as:
\begin{align}
\CB_{u,v,t} = \sqrt{\vecD_{u}} \vert \vert x_{v} \vert \vert_{2}
\end{align}
The $\ell_{2}$ norms for all the target nodes $\tar$ can be pre-computed. If we maintain the $\vecD$ vector, the confidence intervals for all pairs can be computed in $O(n^2)$ time. 

Note that $\vecD_{t}$ requires $O(n)$ storage and can be updated across rounds in $O(K)$ time using the Sherman Morrison formula. Specifically, if $\vecD_{\so,t}$ refers to the $\so^{th}$ element in the vector $\vecD_{t}$, then
\[
\vecD_{\so,t+1} = 
\begin{dcases}
\frac{\vecD_{\so,t}}{(1 + \vecD_{\so,t})}, & \text{if} \so \in \seeds_{t} \\
\vecD_{\so, t}, & \text{otherwise}
\end{dcases}
\]
Hence, the total complexity of implementing Laplacian regularization is $O(d n^2)$. We need to store the $\vecw$ vector, the Laplacian and the diagonal vectors $\beta$ and $\vecD$. Hence, the total memory requirement is $O(dn + m)$. 

\end{document}